\newcommand*{\CVPR}{}

\newcommand*{\CAMREADY}{}

\ifdefined\NIPS
	\documentclass{article}
	\usepackage{nips15submit_e,times}
	\usepackage{hyperref}
	\usepackage{url}
	\usepackage[square,comma,numbers]{natbib}
\fi
\ifdefined\CVPR
	\documentclass[10pt,twocolumn,letterpaper]{article}
	\usepackage{cvpr}
	\usepackage{times}
	\usepackage{epsfig}
	\usepackage{graphicx}
	\usepackage{amsmath}
	\usepackage{amssymb}
	\usepackage[square,comma,numbers]{natbib}
\fi
\ifdefined\AISTATS
	\documentclass[twoside]{article}
	\ifdefined\CAMREADY
		\usepackage[accepted]{aistats2015}
	\else
		\usepackage{aistats2015}
	\fi
	\usepackage{url}
	\usepackage[round,authoryear]{natbib}
\fi
\ifdefined\ICML
	\documentclass{article}
	\usepackage{times}
	\usepackage{graphicx}
	\usepackage{subfigure}
	\usepackage{algorithm}
	\usepackage{algorithmic}
	\usepackage{hyperref}
	
	\ifdefined\CAMREADY
		\usepackage[accepted]{icml2016}
	\else
		\usepackage{icml2016}
	\fi
	\usepackage{natbib}
\fi

\usepackage{color}
\usepackage{amsfonts}
\usepackage{amsmath}
\usepackage{algorithm}
\usepackage{algorithmic}
\usepackage{graphicx}
\usepackage{nicefrac}
\usepackage{bbm}
\usepackage{amsthm}
\usepackage{wrapfig}
\usepackage{tikz}
\usepackage[titletoc,title]{appendix}

\newtheorem{lemma}{Lemma}

\newtheorem{claim}{Claim}

\def\be{\begin{equation}}
\def\ee{\end{equation}}
\def\beas{\begin{eqnarray*}}
\def\eeas{\end{eqnarray*}}
\def\bea{\begin{eqnarray}}
\def\eea{\end{eqnarray}}

\newcommand{\x}{{\mathbf x}}

\newcommand{\uu}{{\mathbf u}}
\newcommand{\vv}{{\mathbf v}}
\newcommand{\w}{{\mathbf w}}

\newcommand{\e}{{\mathbf e}}
\newcommand{\aaa}{{\mathbf a}}
\newcommand{\bb}{{\mathbf b}}

\newcommand{\oo}{{\mathbf o}}

\newcommand{\1}{{\mathbf 1}}
\newcommand{\0}{{\mathbf 0}}
\newcommand{\A}{{\mathcal A}}
\newcommand{\B}{{\mathcal B}}

\newcommand{\E}{{\mathcal E}}
\newcommand{\F}{{\mathcal F}}

\newcommand{\OO}{{\mathcal O}}

\newcommand{\R}{{\mathbb R}}
\newcommand{\N}{{\mathbb N}}
\newcommand{\alphabf}{{\boldsymbol{\alpha}}}

\newcommand{\epsbf}{{\boldsymbol{\epsilon}}}

\newcommand{\indc}[1]{\mathbbm{1}\left[#1\right]}

\newcommand{\inprod}[2]  {\left\langle{#1},{#2}\right\rangle}

\DeclareMathOperator*{\argmax}{argmax} 
  
\DeclareMathOperator*{\mean}{mean}

\ifdefined\CAMREADY
	\newcommand{\simnets}{cohen2014simnets}
	\newcommand{\deepsimnets}{cohen2015deep}
	\newcommand{\expresstensors}{cohen2015expressive}
\else
	\newcommand{\simnets}{anonymous}
	\newcommand{\deepsimnets}{anonymous}
	\newcommand{\expresstensors}{anonymous}
\fi
\newcommand{\otimesg}{\otimes_g}

\newcommand{\odotg}{\odot_g}

\ifdefined\CVPR
	\usepackage[breaklinks=true,bookmarks=false]{hyperref}
	\ifdefined\CAMREADY
		\cvprfinalcopy 
	\fi
	\def\cvprPaperID{****} 
	
\fi
\ifdefined\NIPS

	\ifdefined\CAMREADY
		\nipsfinalcopy
	\fi
\fi
\ifdefined\ICML
	\newcommand{\eg}{\emph{e.g.}}
	\newcommand{\ie}{\emph{i.e.}}
	
	\newcommand{\etc}{\emph{etc.}}
	\newcommand{\vs}{\emph{vs.}}
	\newcommand{\wrt}{w.r.t.}

\fi

\begin{document}

\ifdefined\NIPS
	\title{Convolutional Rectifier Networks as Generalized Tensor Decompositions}
	\author{
	Nadav Cohen \\
	The Hebrew University of Jerusalem \\
	\texttt{cohennadav@cs.huji.ac.il} \\
	\And 
	Amnon Shashua \\
	The Hebrew University of Jerusalem \\
	\texttt{shashua@cs.huji.ac.il} \\
	}
	\maketitle
\fi
\ifdefined\CVPR
	\title{Convolutional Rectifier Networks as Generalized Tensor Decompositions}
	\author{
	Nadav Cohen \\
	The Hebrew University of Jerusalem \\	
	\texttt{cohennadav@cs.huji.ac.il} \\
	\and
	Amnon Shashua \\
	The Hebrew University of Jerusalem \\
	\texttt{shashua@cs.huji.ac.il} \\
	}
	\maketitle
\fi
\ifdefined\AISTATS
	\twocolumn[
	\aistatstitle{Convolutional Rectifier Networks as Generalized Tensor Decompositions}
	\ifdefined\CAMREADY
		\aistatsauthor{Nadav Cohen \And Amnon Shashua}
		\aistatsaddress{The Hebrew University of Jerusalem \And The Hebrew University of Jerusalem}
	\else
		\aistatsauthor{Anonymous Author 1 \And Anonymous Author 2}
		\aistatsaddress{Unknown Institution 1 \And Unknown Institution 2}
	\fi
	]	
\fi
\ifdefined\ICML
	\icmltitlerunning{Convolutional Rectifier Networks as Generalized Tensor Decompositions}
	\twocolumn[
	\icmltitle{Convolutional Rectifier Networks as Generalized Tensor Decompositions}
	\icmlauthor{Nadav Cohen}{cohennadav@cs.huji.ac.il}
	\icmladdress{The Hebrew University of Jerusalem}
	\icmlauthor{Amnon Shashua}{shashua@cs.huji.ac.il}
	\icmladdress{The Hebrew University of Jerusalem}
	\icmlkeywords{deep learning, convolutional neural network, tensor decomposition, expressive power}
	\vskip 0.3in
	]
\fi

\begin{abstract}

Convolutional rectifier networks, i.e. convolutional neural networks with rectified linear activation and max or average pooling, are the cornerstone of modern deep learning. 
However, despite their wide use and success, our theoretical understanding of the expressive properties that drive these networks is partial at best.
On the other hand, we have a much firmer grasp of these issues in the world of arithmetic circuits.
Specifically, it is known that convolutional arithmetic circuits possess the property of "complete depth efficiency", meaning that besides a negligible set, all functions that can be implemented by a deep network of polynomial size, require exponential size in order to be implemented (or even approximated) by a shallow network.

In this paper we describe a construction based on generalized tensor decompositions, that transforms convolutional arithmetic circuits into convolutional rectifier networks.
We then use mathematical tools available from the world of arithmetic circuits to prove new results. 
First, we show that convolutional rectifier networks are universal with max pooling but not with average pooling.
Second, and more importantly, we show that depth efficiency is weaker with convolutional rectifier networks than it is with convolutional arithmetic circuits.
This leads us to believe that developing effective methods for training convolutional arithmetic circuits, thereby fulfilling their expressive potential, may give rise to a deep learning architecture that is provably superior to convolutional rectifier networks but has so far been overlooked by practitioners.

\end{abstract}

\section{Introduction} \label{sec:intro}

Deep neural networks are repeatedly proving themselves to be extremely effective machine learning models, providing state of the art accuracies on a wide range of tasks (see~\cite{LeCun:2015dt,Goodfellow-et-al-2016-Book}).
Arguably, the most successful deep learning architecture to date is that of convolutional neural networks (\emph{ConvNets},~\cite{lecun1995convolutional}), which prevails in the field of computer vision, and is recently being harnessed for many other application domains as well (\eg~\cite{shen2014learning,wallach2015atomnet,clark2014teaching}).
Modern ConvNets are formed by stacking layers one after the other, where each layer consists of a linear convolutional operator followed by Rectified Linear Unit (\emph{ReLU}~\cite{nair2010rectified}) activation ($\sigma(z)=\max\{0,z\}$), which in turn is followed by max or average pooling ($P\{c_j\}=\max\{c_j\}$ or $P\{c_j\}=\mean\{c_j\}$ respectively).
Such models, which we refer to as \emph{convolutional rectifier networks}, have driven the resurgence of deep learning~(\cite{Krizhevsky:2012wl}), and represent the cutting edge of the ConvNet architecture~(\cite{Szegedy:2014tb,simonyan2014very}).

Despite their empirical success, and the vast attention they are receiving, our theoretical understanding of convolutional rectifier networks is partial at best.
It is believed that they enjoy \emph{depth efficiency}, \ie that when allowed to go deep, such networks can implement with polynomial size computations that would require super-polynomial size if the networks were shallow.
However, formal arguments that support this are scarce.
It is unclear to what extent convolutional rectifier networks leverage depth efficiency, or more formally, what is the proportion of weight settings that would lead a deep network to implement a computation that cannot be efficiently realized by a shallow network.
We refer to the most optimistic situation, where this takes place for all weight settings but a negligible (zero measure) set, as \emph{complete depth efficiency}.

Compared to convolutional rectifier networks, our theoretical understanding of depth efficiency for arithmetic circuits, and in particular for convolutional arithmetic circuits, is much more developed.
\emph{Arithmetic circuits} (also known as Sum-Product Networks,~\cite{Poon-Domingos2011}) are networks with two types of nodes: sum nodes, which compute a weighted sum of their inputs, and product nodes, computing the product of their inputs.
The depth efficiency of arithmetic circuits has been studied by the theoretical computer science community for the last five decades, long before the resurgence of deep learning.
Although many problems in the area remain open, significant progress has been made over the years, making use of various mathematical tools.
\emph{Convolutional arithmetic circuits} form a specific sub-class of arithmetic circuits.
Namely, these are ConvNets with linear activation ($\sigma(z)=z$) and product pooling ($P\{c_j\}=\prod{c_j}$).
Recently, \cite{\expresstensors}~analyzed convolutional arithmetic circuits through tensor decompositions, essentially proving, for the type of networks considered, that \emph{depth efficiency holds completely}.
Although convolutional arithmetic circuits are known to be equivalent to SimNets~(\cite{\simnets}), a new deep learning architecture that has recently demonstrated promising empirical performance~(\cite{\deepsimnets}), they are fundamentally different from convolutional rectifier networks.
Accordingly, the result established in~\cite{\expresstensors} does not apply to the models most commonly used in practice.

In this paper we present a construction, based on the notion of \emph{generalized tensor decompositions}, that transforms convolutional arithmetic circuits of the type described in~\cite{\expresstensors} into convolutional rectifier networks. 
We then use the available mathematical tools from the world of arithmetic circuits to prove new results concerning the expressive power and depth efficiency of convolutional rectifier networks.
Namely, we show that with ReLU activation, average pooling leads to loss of universality, whereas max pooling is universal but enjoys depth efficiency to a lesser extent than product pooling with linear activation (convolutional arithmetic circuits).
These results indicate that from the point of view of expressive power and depth efficiency, convolutional arithmetic circuits (SimNets) have an advantage over the prevalent convolutional rectifier networks (ConvNets with ReLU activation and max or average pooling).
This leads us to believe that developing effective methods for training convolutional arithmetic circuits, thereby fulfilling their expressive potential, may give rise to a deep learning architecture that is provably superior to convolutional rectifier networks but has so far been overlooked by practitioners.

The remainder of the paper is organized as follows.  
In sec.~\ref{sec:related_work} we review existing works relating to depth efficiency of arithmetic circuits and networks with ReLU activation.
Sec.~\ref{sec:gen_decomp} presents our definition of generalized tensor decompositions, followed by sec.~\ref{sec:nets2tens} which employs this concept to frame convolutional rectifier networks.
In sec.~\ref{sec:analysis} we make use of this framework for an analysis of the expressive power and depth efficiency of such networks.
Finally, sec.~\ref{sec:discussion} concludes.

\section{Related Work} \label{sec:related_work}

The literature on the computational complexity of arithmetic circuits is far too wide to cover here, dating back over five decades.
Although many of the fundamental questions in the field remain open, significant progress has been made over the years, developing and employing a vast share of mathematical tools from branches of geometry, algebra, analysis, combinatorics, and more.
We refer the interested reader to~\cite{shpilka2010arithmetic} for a survey written in 2010, and mention here the more recent works~\cite{bengio2011shallow} and~\cite{martens2014expressive} studying depth efficiency of arithmetic circuits in the context of deep learning (Sum-Product Networks). 
Compared to arithmetic circuits, the literature on depth efficiency of neural networks with ReLU activation is far less developed, primarily since these models were only introduced several years ago~(\cite{nair2010rectified}).
There have been some notable works on this line, but these employ dedicated mathematical machinery, not making use of the plurality of available tools from the world of arithmetic circuits.
\cite{pascanu2013number}~and~\cite{montufar2014number} use combinatorial arguments to characterize the maximal number of linear regions in functions generated by ReLU networks, thereby establishing existence of depth efficiency.
\cite{telgarsky2016benefits}~uses semi-algebraic geometry to analyze the number of oscillations in functions realized by neural networks with semi-algebraic activations, ReLU in particular.
The fundamental result proven in~\cite{telgarsky2016benefits} is the existence, for every $k\in\N$, of functions realizable by networks with $\Theta(k^3)$ layers and $\Theta(1)$ nodes per layer, which cannot be approximated by networks with $\OO(k)$ layers unless these are exponentially large (have $\Omega(2^k)$ nodes).
The work in~\cite{eldan2015power} makes use of Fourier analysis to show existence of functions that are efficiently computable by depth-3 networks, yet require exponential size in order to be approximated by depth-2 networks.
The result applies to various activations, including ReLU.
\cite{poggio2015theory}~also compares the computational abilities of deep \vs shallow networks under different activations that include ReLU.
However, the complexity measure considered in~\cite{poggio2015theory} is the VC dimension, whereas our interest lies in network size.

None of the analyses above account for convolutional networks
\footnote{
By this we mean that in all analyses, the deep networks shown to benefit from depth (\ie to realize functions that require super-polynomial size from shallow networks) are not ConvNets.
},
thus they do not apply to the deep learning architecture most commonly used in practice.
Recently, \cite{\expresstensors}~introduced convolutional arithmetic circuits, which may be viewed as ConvNets with linear activation and product pooling.
These networks were shown to correspond to hierarchical tensor decompositions (see~\cite{Hackbusch-book}).
Tools from linear algebra, functional analysis and measure theory were then employed to prove that the networks are universal, and exhibit \emph{complete depth efficiency}.
Although similar in structure, convolutional arithmetic circuits are inherently different from convolutional rectifier networks (ConvNets with ReLU activation and max or average pooling).
Accordingly, the analysis carried out in~\cite{\expresstensors} does not apply to the networks at the forefront of deep learning.

Closing the gap between the networks analyzed in~\cite{\expresstensors} and convolutional rectifier networks is the topic of this paper.
We achieve this by generalizing tensor decompositions, thereby opening the door to mathematical machinery as used in~\cite{\expresstensors}, harnessing it to analyze, for the first time, the depth efficiency of convolutional rectifier networks.

\section{Generalized Tensor Decompositions} \label{sec:gen_decomp}

We begin by establishing basic tensor-related terminology and notations.
\footnote{
The definitions we give are actually concrete special cases of more abstract algebraic definitions as given in~\cite{Hackbusch-book}.
We limit the discussion to these special cases since they suffice for our needs and are easier to grasp.
}
For our purposes, a \emph{tensor} is simply a multi-dimensional array:
$$\A_{d_1,\ldots,d_N}\in\R\quad ,d_i \in [M_i]$$
The \emph{order} of a tensor is defined to be the number of indexing entries in the array, which are referred to as \emph{modes}.  
The term \emph{dimension} stands for the number of values an index can take in a particular mode.  
For example, the tensor $\A$ above has order $N$ and dimension $M_i$ in mode $i$, $i\in[N]$.  
The space of all possible configurations $\A$ can take is called a \emph{tensor space} and is denoted, quite naturally, by $\R^{M_1{\times\cdots\times}M_N}$.

The fundamental operator in tensor analysis is the \emph{tensor product}, denoted by $\otimes$.
It is an operator that intakes two tensors $\A\in\R^{M_1{\times\cdots\times}M_P}$ and $\B\in\R^{M_{P+1}{\times\cdots\times}M_{P+Q}}$ (orders $P$ and $Q$ respectively), and returns a tensor $\A\otimes\B\in\R^{M_1{\times\cdots\times}M_{P+Q}}$ (order $P+Q$) defined by:
\be
\left(\A\otimes\B\right)_{d_1,\ldots,d_{P+Q}}=\A_{d_1,\ldots,d_P}\cdot\B_{d_{P+1},\ldots,d_{P+Q}}
\label{eq:tensor_prod}
\ee
Notice that in the case $P=Q=1$, the tensor product reduces to the standard outer product between vectors, \ie if $\uu\in\R^{M_1}$ and $\vv\in\R^{M_2}$, then $\uu\otimes\vv$ is no other than the rank-1 matrix $\uu\vv^\top\in\R^{M_1{\times}M_2}$.

\emph{Tensor decompositions} (see~\cite{Kolda-Bader2009} for a survey) may be viewed as schemes for expressing tensors using tensor products and weighted sums.
For example, suppose we have a tensor $\A\in\R^{M_1{\times\cdots\times}M_N}$ given by:
$$\A=\sum_{j_1{\ldots}j_N=1}^{J}c_{j_1{\ldots}j_N}\cdot\aaa^{j_1,1}\otimes\cdots\otimes\aaa^{j_N,N}$$
This expression is known as a Tucker decomposition, parameterized by the coefficients $\{c_{j_1{\ldots}j_N}\in\R\}_{j_1{\ldots}j_N\in[J]}$ and vectors $\{\aaa^{j,i}\in\R^{M_i}\}_{i\in[N],j\in[J]}$. 
It is different from the \emph{CP} (rank-$1$) and \emph{Hierarchical Tucker} decompositions our analysis will rely upon (see sec.~\ref{sec:nets2tens}).
All decompositions however are closely related, specifically in the fact that they are based on iterating between tensor products and weighted sums.

Our construction and analysis are facilitated by generalizing the tensor product, which in turn generalizes tensor decompositions. 
For an associative and commutative binary operator $g$, \ie a function $g:\R\times\R\to\R$ such that $\forall{a,b,c\in\R}:g(g(a,b),c)=g(a,g(b,c))$ and $\forall{a,b\in\R}:g(a,b)=g(b,a)$, the \emph{generalized tensor product} $\otimesg$, an operator intaking tensors $\A\in\R^{M_1{\times\cdots\times}M_P},\B\in\R^{M_{P+1}{\times\cdots\times}M_{P+Q}}$ and returning tensor $\A\otimesg\B\in\R^{M_1{\times\cdots\times}M_{P+Q}}$, is defined as follows:
\be
\left(\A\otimesg\B\right)_{d_1,\ldots,d_{P+Q}}=g(\A_{d_1,\ldots,d_P},\B_{d_{P+1},\ldots,d_{P+Q}})
\label{eq:gen_tensor_prod}
\ee
\emph{Generalized tensor decompositions} are simply obtained by plugging in the generalized tensor product $\otimesg$ in place of the standard tensor product $\otimes$.

\section{From Networks to Tensors} \label{sec:nets2tens}

\begin{figure*}
\includegraphics[width=\textwidth]{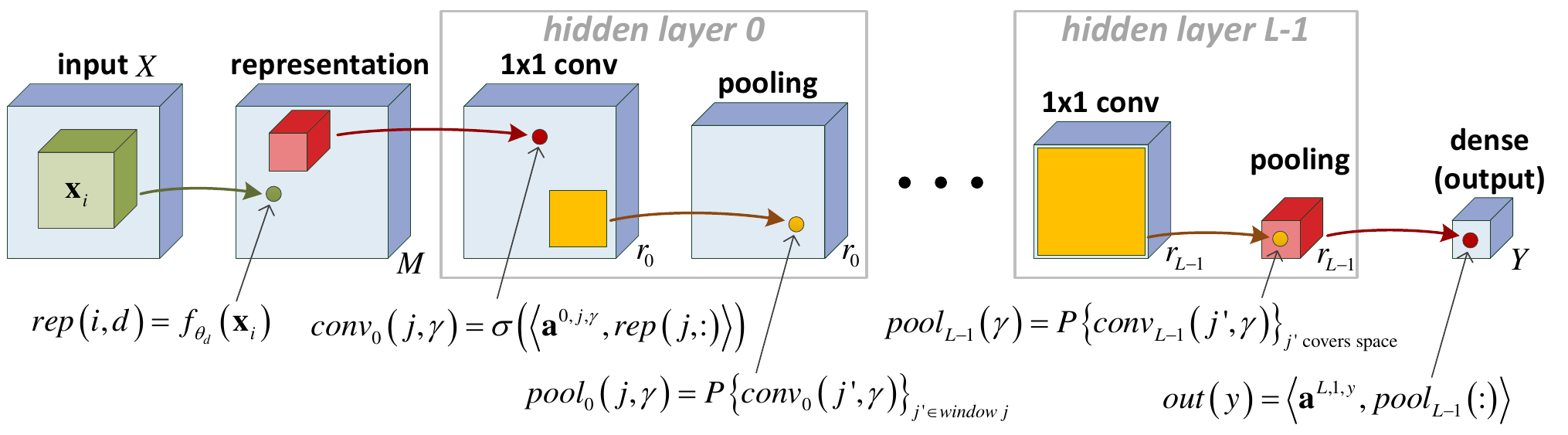}
\caption{
ConvNet architecture analyzed in this paper.  
The representation convolves functions $f_{\theta_d}(\cdot)$ across input patches (a standard convolutional layer is obtained by setting $f_{\theta_d}(\x)=\sigma(\w_d^\top\x+b_d)$).
$L$ hidden layers follow, each comprising $1\times1$ convolution (optionally without spatial weight sharing) followed by activation $\sigma(\cdot)$ and pooling $P(\cdot)$.
The last hidden layer reduces feature maps to singletons, and these are mapped to network outputs through a dense linear layer.
Convolutional arithmetic circuits as analyzed in~\cite{\expresstensors} correspond to linear activation ($\sigma(z)=z$) and product pooling ($P\{c_j\}=\prod{c_j}$).
Convolutional rectifier networks are obtained by setting ReLU activation ($\sigma(z)=\max\{0,z\}$) and max or average pooling ($P\{c_j\}=\max\{c_j\}$ or $P\{c_j\}=\mean\{c_j\}$ respectively).
Best viewed in color.
}
\label{fig:convnet}
\end{figure*}

The ConvNet architecture analyzed in this paper is presented in fig.~\ref{fig:convnet}.
The input to a network, denoted $X$, is composed of $N$ \emph{patches} $\x_1\ldots\x_N\in\R^s$.
For example, $X$~could represent a $32$-by-$32$ RGB image through $5\times5$ regions crossing the three color bands, in which case, assuming a patch is taken for every pixel (boundaries padded), we have $N=1024$ and $s=75$.
The first layer in a network is referred to as \emph{representation}, and may be thought of as a generalized convolution.
Namely, it consists of applying $M$ \emph{representation functions} $f_{\theta_1}{\ldots}f_{\theta_M}:\R^s\to\R$ to all patches of the input, thereby creating $M$ feature maps.
In the case where the representation functions are standard neurons, \ie $f_{\theta_d}(\x)=\sigma(\w_d^\top\x+b_d)$ for parameters $\theta_d=(\w_d,b_d)\in\R^s\times\R$ and some chosen activation $\sigma(\cdot)$, we obtain a conventional convolutional layer.
More elaborate settings are also possible, for example modeling the representation as a cascade of convolutional layers with pooling in-between.

Following the representation, a network includes $L$ hidden layers indexed by $l=0{\ldots}L-1$.
Each hidden layer $l$ begins with a \emph{$1\times1$ conv} operator, which is simply a 3D convolution with $r_l$ channels and receptive field $1\times1$ followed by point-wise activation $\sigma(\cdot)$.
We allow the convolution to operate without weight sharing, in which case the filters that generate feature maps by sliding across the previous layer may have different coefficients at different spatial locations.
This is often referred to in the deep learning community as a locally-connected layer (see~\cite{Taigman:2014vs}).
We refer to it as the \emph{unshared} case, in contrast to the \emph{shared} case that gives rise to a standard $1\times1$ convolution.
The second (last) operator in a hidden layer is spatial pooling.
Feature maps generated by $1\times1$ conv are decimated, by applying the pooling operator $P(\cdot)$ (\eg max or average) to non-overlapping 2D windows that cover the spatial extent.
The last of the $L$ hidden layers ($l=L-1$) reduces feature maps to singletons (its pooling operator is global), creating a vector of dimension $r_{L-1}$.
This vector is mapped into $Y$ network outputs through a final dense linear layer.

Altogether, the architectural parameters of a ConvNet are the type of representation functions ($f_{\theta_d}$), the pooling window sizes (which in turn determine the number of hidden layers $L$), the setting of conv weights as shared or unshared, the number of channels in each layer ($M$ for representation, $r_0{\ldots}r_{L-1}$ for hidden layers, $Y$ for output), and the choice of activation and pooling operators ($\sigma(\cdot)$ and $P(\cdot)$ respectively).
Given these architectural parameters, the learnable parameters of a network are the representation weights ($\theta_d$), the conv weights ($\aaa^{l,j,\gamma}$ for hidden layer $l$, location $j$ and channel $\gamma$ in the unshared case; $\aaa^{l,\gamma}$ for hidden layer $l$ and channel $\gamma$ in the shared case), and the output weights ($\aaa^{L,1,y}$).

The choice of activation and pooling operators determines the type of network we arrive at.
For linear activation ($\sigma(z)=z$) and product pooling ($P\{c_j\}=\prod{c_j}$) we get a convolutional arithmetic circuit as analyzed in~\cite{\expresstensors}.
For ReLU activation ($\sigma(z)=\max\{0,z\}$) and max or average pooling ($P\{c_j\}=\max\{c_j\}$ or $P\{c_j\}=\mean\{c_j\}$ respectively) we get the commonly used convolutional rectifier networks, on which we focus in this paper.

In terms of pooling window sizes and network depth, we direct our attention to two special cases representing the extremes.
The first is a shallow network that includes global pooling in its single hidden layer~--~see illustration in fig.~\ref{fig:shallow_convnet}.
The second is the deepest possible network, in which all pooling windows cover only two entries, resulting in $L=\log_{2}N$ hidden layers.
These ConvNets, which we refer to as \emph{shallow} and \emph{deep} respectively, will be shown to correspond to canonical tensor decompositions.
It is for this reason, and for simplicity of presentation, that we focus on these special cases.
One may just as well consider networks of intermediate depths with different pooling window sizes, and that would correspond to other, non-standard, tensor decompositions.
The analysis carried out in sec.~\ref{sec:analysis} can easily be adapted to such networks.

In a classification setting, the $Y$ outputs of a network correspond to different categories, and prediction follows the output with highest activation.
Specifically, if we denote by $h_y(\cdot)$ the mapping from network input to output $y$, the predicted label for the instance $X=(\x_1,\ldots,\x_N)\in(\R^s)^N$ is determined by the following classification rule:
$$\hat{y}=\argmax_{y\in[Y]}h_y(X)$$
We refer to $h_y$ as the \emph{score function} of category $y$.
Score functions are studied in this paper through the notion of \emph{grid tensors}.
Given fixed vectors $\x^{(1)}\ldots\x^{(M)}\in\R^s$, referred to as \emph{templates}, the grid tensor of $h_y$, denoted $\A(h_y)$, is defined to be the tensor of order~$N$ and dimension~$M$ in each mode whose entries are given by:
\be
\A(h_y)_{d_1{\ldots}d_N}=h_y(\x^{(d_1)},\ldots,\x^{(d_N)})
\label{eq:grid_tensor}
\ee
That is to say, the grid tensor of a score function under~$M$ templates $\x^{(1)}\ldots\x^{(M)}$, is a tensor of order~$N$ and dimension~$M$ in each mode, holding score values on the exponentially large grid of instances $\{X_{d_1{\ldots}d_N}:=(\x^{(d_1)},\ldots,\x^{(d_N)}):d_1{\ldots}d_N\in[M]\}$.
Before heading on to our analysis of grid tensors generated by ConvNets, to simplify notation, we define $F\in\R^{M{\times}M}$ to be the matrix holding the values taken by the representation functions $f_{\theta_1}{\ldots}f_{\theta_M}:\R^s\to\R$ on the selected templates $\x^{(1)}\ldots\x^{(M)}\in\R^s$:
\bea
F:=
\begin{bmatrix}
f_{\theta_1}(\x^{(1)}) & \cdots & f_{\theta_M}(\x^{(1)}) \\
\vdots & \ddots & \vdots \\
f_{\theta_1}(\x^{(M)}) & \cdots & f_{\theta_M}(\x^{(M)})
\end{bmatrix}
\label{eq:F}
\eea

To express the grid tensor of a ConvNet's score function using generalized tensor decompositions (see sec.~\ref{sec:gen_decomp}), we set the underlying function $g:\R\times\R\to\R$ to be the \emph{activation-pooling} operator defined by:
\be
g(a,b)=P(\sigma(a),\sigma(b))
\label{eq:act_pool_op}
\ee
where $\sigma(\cdot)$ and $P(\cdot)$ are the network's activation and pooling functions, respectively.
Notice that the activation-pooling operator meets the associativity and commutativity requirements under product pooling with linear activation ($g(a,b)=a{\cdot}b$), and under max pooling with ReLU activation ($g(a,b)=\max\{a,b,0\}$).
To account for the case of average pooling with ReLU activation, which a-priori leads to a non-associative activation-pooling operator, we simply replace average by sum, \ie we analyze sum pooling with ReLU activation ($g(a,b)=\max\{a,0\}+\max\{b,0\}$), which from the point of view of expressiveness is completely equivalent to average pooling with ReLU activation (scaling factors can always blend in to linear weights that follow pooling).

With the activation-pooling operator $g$ in place, it is straightforward to see that the grid tensor of $h_y^S$~--~a score function generated by the shallow ConvNet (fig.~\ref{fig:shallow_convnet}), is given by the following generalized tensor decomposition:
\be
\A\left(h_y^S\right) = \sum_{z=1}^Z a_z^y \cdot (F\aaa^{z,1}) \otimesg \cdots \otimesg (F\aaa^{z,N})
\label{eq:gen_cp_decomp}
\ee
$Z$ here is the number of channels in the network's single hidden layer, $\{\aaa^{z,i}\in\R^M\}_{z\in[Z],i\in[N]}$ are the weights in the hidden conv, and $\aaa^y\in\R^Z$ are the weights of output $y$.
The factorization in eq.~\ref{eq:gen_cp_decomp} generalizes the classic CP (CANDECOMP/PARAFAC) decomposition (see~\cite{Kolda-Bader2009} for a historic survey), and we accordingly refer to it as the \emph{generalized CP decomposition}.

Turning to the deep ConvNet (fig.~\ref{fig:convnet} with size-$2$ pooling windows and $L=\log_{2}N$ hidden layers), the grid tensor of its score function $h_y^D$ is given by the hierarchical generalized tensor decomposition below:
\bea
\phi^{1,j,\gamma} &=& \sum_{\alpha=1}^{r_0} a_\alpha^{1,j,\gamma} 
(F\aaa^{0,2j-1,\alpha}) \otimesg  (F\aaa^{0,2j,\alpha})
\nonumber \\
&\cdots& 
\nonumber\\
\phi^{l,j,\gamma} &=& \sum_{\alpha=1}^{r_{l-1}} a_\alpha^{l,j,\gamma} 
\underbrace{\phi^{l-1,2j-1,\alpha}}_{\text{order $2^{l-1}$}} \otimesg  
\underbrace{\phi^{l-1,2j,\alpha}}_{\text{order $2^{l-1}$}} 
\nonumber\\
&\cdots& 
\nonumber\\
\phi^{L-1,j,\gamma} &=& \sum_{\alpha=1}^{r_{L-2}} a_\alpha^{L-1,j,\gamma} 
\underbrace{\phi^{L-2,2j-1,\alpha}}_{\text{order $\frac{N}{4}$}} \otimesg  
\underbrace{\phi^{L-2,2j,\alpha}}_{\text{order $\frac{N}{4}$}}  
\nonumber\\ 
\A\left(h_y^D\right) &=& \sum_{\alpha=1}^{r_{L-1}} a_\alpha^{L,1,y} 
\underbrace{\phi^{L-1,1,\alpha}}_{\text{order $\frac{N}{2}$}} \otimesg  
\underbrace{\phi^{L-1,2,\alpha}}_{\text{order $\frac{N}{2}$}}  
\label{eq:gen_ht_decomp} 
\eea
$r_0{\ldots}r_{L-1}\in\N$ here are the number of channels in the network's hidden layers, $\{\aaa^{0,j,\gamma}\in\R^M\}_{j\in[N],\gamma\in[r_0]}$ are the weights in the first hidden conv, $\{\aaa^{l,j,\gamma} \in \R^{r_{l-1}}\}_{l\in[L-1],j\in[N/2^l],\gamma\in[r_l]}$ are the weights in the following hidden convs, and $\aaa^{L,1,y}\in\R^{r_{L-1}}$ are the weights of output $y$.
The factorization in eq.~\ref{eq:gen_ht_decomp} generalizes the Hierarchical Tucker decomposition introduced in~\cite{Hackbusch:2009jj}, and is accordingly referred to as the \emph{generalized HT decomposition}.

\begin{figure}
\includegraphics[width=\columnwidth]{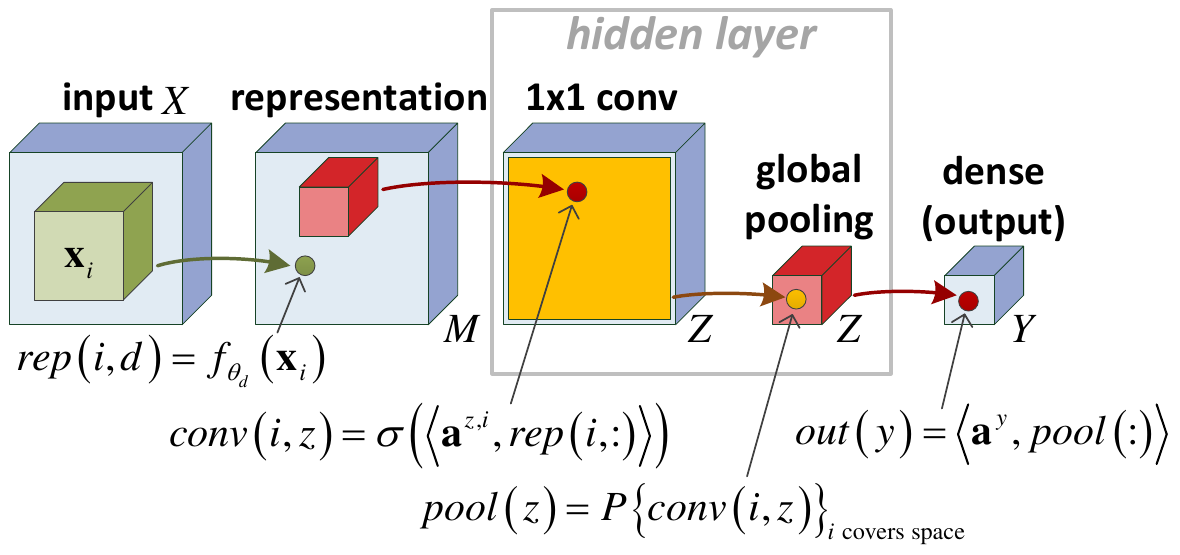}
\caption{
Shallow ConvNet with global pooling in its single hidden layer.
Best viewed in color.
}
\label{fig:shallow_convnet}
\end{figure}

To conclude this section, we presented a ConvNet architecture (fig.~\ref{fig:convnet}) whose activation and pooling operators may be chosen to realize convolutional arithmetic circuits (linear activation, product pooling) or convolutional rectifier networks (ReLU activation, max/average pooling).
We then defined the grid tensor of a network's score function as a tensor holding function values on an exponentially large grid whose points are sequences with elements chosen from a finite set of templates.
Then, we saw that the grid tensor of a shallow ConvNet (fig.~\ref{fig:shallow_convnet}) is given by the generalized CP decomposition (eq.~\ref{eq:gen_cp_decomp}), and a grid tensor of a deep ConvNet (fig.~\ref{fig:convnet} with $L=\log_{2}N$) is given by the generalized HT decomposition (eq.~\ref{eq:gen_ht_decomp}).
In the next section we utilize the connection between ConvNets and generalized tensor decompositions for an analysis of the expressive power and depth efficiency of convolutional rectifier networks.

\section{Capacity Analysis} \label{sec:analysis}

In this section we analyze score functions expressible by the shallow and deep ConvNets (fig.~\ref{fig:shallow_convnet}, and fig.~\ref{fig:convnet} with $L=\log_{2}N$, respectively) under ReLU activation with max or average pooling (convolutional rectifier networks), comparing these settings against linear activation with product pooling (convolutional arithmetic circuits).
Score functions are analyzed through grid tensors (eq.~\ref{eq:grid_tensor}), represented by the generalized tensor decompositions established in the previous section: the generalized CP decomposition (eq.~\ref{eq:gen_cp_decomp}) corresponding to the shallow network, and the generalized HT decomposition (eq.~\ref{eq:gen_ht_decomp}) corresponding to the deep network.
The analysis is organized as follows.
In sec.~\ref{sec:analysis:prelim} we present preliminary material required in order to follow our proofs.
Sec.~\ref{sec:analysis:temp_rep_funcs} discusses templates and representation functions, which form the bridge between score functions and generalized tensor decompositions.
Sec.~\ref{sec:analysis:matricization} presents matricization~--~a technical tool that facilitates the use of matrix theory for analyzing generalized tensor decompositions.
The actual analysis begins in sec.~\ref{sec:analysis:universality}, where we address the question of universality, \ie of the ability of networks to realize any score function when their size is unlimited.
This is followed by sec.~\ref{sec:analysis:depth_eff} which studies depth efficiency, namely, situations where functions efficiently computable by deep networks require shallow networks to have super-polynomial size.
Finally, sec.~\ref{sec:analysis:shared_coeff} analyzes the case of coefficient sharing, in which the conv operators of our networks are standard convolutions (as opposed to the more general locally-connected layers).

\subsection{Preliminaries} \label{sec:analysis:prelim}

For evaluating the completeness of depth efficiency, and for other purposes as well, we are often interested in the ``volume'' of sets in a Euclidean space, or more formally, in their Lebesgue measure.
While an introduction to Lebesgue measure theory is beyond the scope of this paper (the interested reader is referred to~\cite{jones2001lebesgue}), we restate here several concepts and results our proofs will rely upon.
A zero measure set can intuitively be thought of as having zero volume.
A union of countably many zero measure sets is itself a zero measure set.
If we randomize a point in space by some continuous distribution, the probability of hitting a zero measure set is always zero.
A useful fact (proven in~\cite{caron2005zero} for example) is that the zero set of a polynomial, \ie the set of points on which a polynomial vanishes, is either the entire space (when the polynomial in question is the zero polynomial), or it must have measure zero.
An open set always has positive measure, and when a point in space is drawn by a continuous distribution with non-vanishing continuous probability density function, the probability of hitting such a set is positive.

Apart from measure theory, we will also be using tools from the field of tensor analysis.
Here too, a full introduction to the topic is beyond our scope (we refer the interested reader to~\cite{Hackbusch-book}), and we only list some concepts and results that will be used.
First, a fact that relates to abstract tensor products over function spaces is the following.
If $f_{\theta_1}{\ldots}f_{\theta_M}:\R^s\to\R$ are linearly independent functions, then the product functions $\{(\x^{(1)},\ldots,\x^{(M)})\mapsto\prod_{i=1}^{M}f_{\theta_{d_i}}(\x^{(i)})\}_{d_1{\ldots}d_M\in[M]}$ from $(\R^s)^M$ to $\R$ are linearly independent as well.
Back to tensors as we have defined them (multi-dimensional arrays), a very important concept is that of \emph{rank}, which for order-$2$ tensors reduces to the standard notion of matrix rank.
A tensor is said to have rank $1$ if it may be written as a tensor product between non-zero vectors ($\A=\vv^1\otimes\cdots\otimes\vv^N$).
The rank of a general tensor is defined to be the minimal number of rank-$1$ tensors that may be summed up to produce it.
A useful fact is that the rank of an order-$N$ tensor with dimension $M_i$ in each mode $i\in[N]$, is no greater than $\prod_{i}M_i/\max_{i}M_i$.
On the other hand, all such tensors, besides a zero measure set, have rank equal to at least $\min\{\prod_{i~even}M_i,\prod_{i~odd}M_i\}$.
As in the special case of matrices, the rank is sub-additive, \ie $rank(\A+\B){\leq}rank(\A)+rank(\B)$ for any tensors $\A,\B$ of matching dimensions.
The rank is sub-multiplicative \wrt the tensor product, \ie $rank(\A\otimes\B){\leq}rank(\A){\cdot}rank(\B)$ for any tensors $\A,\B$.
Finally, we use the fact that permuting the modes of a tensor does not alter its rank.

\subsection{Templates and Representation Functions} \label{sec:analysis:temp_rep_funcs}

The expressiveness of our ConvNets obviously depends on the possible forms that may be taken by the representation functions $f_{\theta_1}{\ldots}f_{\theta_M}:\R^s\to\R$.
For example, if representation functions are limited to be constant, the ConvNets can only realize constant score functions.
We denote by $\F:=\{f_\theta:\R^s\to\R:\theta\in\Theta\}$ the parametric family from which representation functions are chosen, and make two mild assumptions on this family:
\begin{itemize}
\item \textbf{Continuity}: 
$f_\theta(\x)$ is continuous \wrt both $\theta$ and $\x$.
\item \textbf{Non-degeneracy}:
For any $\x^{(1)}\ldots\x^{(M)}\in\R^s$ such that $\x_i\neq\x_j~\forall{i{\neq}j}$, there exist $f_{\theta_1}{\ldots}f_{\theta_M}\in\F$ for which the matrix $F$ defined in eq.~\ref{eq:F} is non-singular.
\end{itemize}
Both of the assumptions above are met for most reasonable choices of $\F$.
In particular, non-degeneracy holds when representation functions are standard neurons:

\begin{claim} \label{claim:F_inv_f4x}
The parametric family:
\be
\F=\left\{f_\theta(\x)=\sigma(\w^\top\x+b):\theta=(\w,b)\in\R^s\times\R\right\}
\label{eq:rep_neurons}
\ee
where $\sigma(\cdot)$ is any sigmoidal activation
\footnote{
$\sigma(\cdot)$ is sigmoidal if it is monotonic with $\lim_{z\to-\infty}\sigma(z)=c$ and $\lim_{z\to+\infty}\sigma(z)=C$ for some $c{\neq}C$ in $\R$.
}
or the ReLU activation, meets the non-degeneracy condition (\ie for any distinct $\x^{(1)}\ldots\x^{(M)}\in\R^s$ there exist $f_{\theta_1}{\ldots}f_{\theta_M}\in\F$ such that the matrix $F$ defined in eq.~\ref{eq:F} is non-singular).
\end{claim}

\begin{proof}
We first show that given distinct $\x^{(1)}\ldots\x^{(M)}\in\R^s$, there exists a vector $\w\in\R^s$ such that $\w^\top\x^{(i)}\neq\w^\top\x^{(j)}$ for all $1{\leq}i<j{\leq}M$.
$\w$ satisfies this condition if it is not perpendicular to any of the finitely many non-zero vectors $\{\x^{(i)}-\x^{(j)}:1{\leq}i<j{\leq}M\}$.
If for every $1{\leq}i<j{\leq}M$ we denote by $P^{(i,j)}\subset\R^s$ the set of points perpendicular to $\x^{(i)}-\x^{(j)}$, we obtain that $\w$ satisfies the desired condition if it does not lie in the union $\bigcup_{1{\leq}i<j{\leq}M}P^{(i,j)}$.
Each $P^{(i,j)}$ is the zero set of a non-zero polynomial, and in particular has measure zero.
The finite union $\bigcup_{1{\leq}i<j{\leq}M}P^{(i,j)}$ thus has measure zero as well, and accordingly cannot cover the entire space.
This implies that $\w\in\R^s\setminus\bigcup_{1{\leq}i<j{\leq}M}P^{(i,j)}$ indeed exists.

Assume without loss of generality $\w^\top\x^{(1)}<\ldots<\w^\top\x^{(M)}$.
We may then choose $b_1{\ldots}b_M\in\R$ such that $-\w^\top\x^{(M)}<b_M<\ldots<-\w^\top\x^{(1)}<b_1$.
For $i,j\in[M]$, $\w^\top\x^{(i)}+b_j$ is positive when $j{\leq}i$ and negative when $j>i$.
Therefore, if $\sigma(\cdot)$ is chosen as the ReLU activation, defining $f_{\theta_j}(\x)=\sigma(\w^\top\x+b_j)$ for every $j\in[M]$ gives rise to a matrix $F$ (eq.~\ref{eq:F}) that is lower triangular with non-zero values on its diagonal.
This proves the desired result for the case of ReLU activation.

Consider now the case of sigmoidal activation, where $\sigma(\cdot)$ is monotonic with $\lim_{z\to-\infty}\sigma(z)=c$ and $\lim_{z\to+\infty}\sigma(z)=C$ for some $c{\neq}C$ in $\R$.
Letting $\w\in\R^s$ and $b_1{\ldots}b_M\in\R$ be as above, we introduce a scaling factor $\alpha>0$, and define $f_{\theta_j}(\x)=\sigma(\alpha\w^\top\x+\alpha b_j)$ for every $j\in[M]$.
It is not difficult to see that as $\alpha\to+\infty$, the matrix $F$ tends closer and closer to a matrix holding $C$ on and below its diagonal, and $c$ elsewhere.
The latter matrix is non-singular, and in particular has non-zero determinant $d\neq0$.
The determinant of $F$ converges to $d$ as $\alpha\to+\infty$, so for large enough $\alpha$, $F$ is non-singular.
\end{proof}

Non-degeneracy means that given distinct templates, one may choose representation functions for which $F$ is non-singular.
We may as well consider the opposite situation, where we are given representation functions, and would like to choose templates leading to non-singular $F$.  Apparently, so long as the representation functions are linearly independent, this is always possible:

\begin{claim} \label{claim:F_inv_x4f}
Let $f_{\theta_1}{\ldots}f_{\theta_M}:\R^s\to\R$ be any linearly independent continuous functions.
Then, there exist $\x^{(1)}\ldots\x^{(M)}\in\R^s$ such that $F$ (eq.~\ref{eq:F}) is non-singular.
\end{claim}

\begin{proof}
We may view the determinant of $F$ (eq.~\ref{eq:F}) as a function of $(\x^{(1)},\ldots,\x^{(M)})$:
$$\det{F}(\x^{(1)},\ldots,\x^{(M)})=\sum_{\delta{\in}S_M}sign(\delta)\prod_{i=1}^{M}f_{\theta_{\delta(i)}}(\x^{(i)})$$
where $S_M$ stands for the permutation group on $[M]$, and $sign(\delta)\in\{\pm1\}$ is the sign of the permutation~$\delta$.
This in particular shows that $\det{F}(\x^{(1)},\ldots,\x^{(M)})$ is a non-zero linear combination of the product functions $\{(\x^{(1)},\ldots,\x^{(M)})\mapsto\prod_{i=1}^{M}f_{\theta_{d_i}}(\x^{(i)})\}_{d_1{\ldots}d_M\in[M]}$.
Since these product functions are linearly independent (see sec.~\ref{sec:analysis:prelim}), $\det{F}(\x^{(1)},\ldots,\x^{(M)})$ cannot be the zero function.
That is to say, there exist $\x^{(1)}\ldots\x^{(M)}\in\R^s$ such that $\det{F}(\x^{(1)},\ldots,\x^{(M)})\neq0$.
\end{proof}

As stated previously, the analysis carried out in this paper studies score functions expressible by ConvNets through the notion of grid tensors.
The translation of score functions into grid tensors is facilitated by the choice of templates $\x^{(1)}\ldots\x^{(M)}\in\R^s$ (eq.~\ref{eq:grid_tensor}).
For general templates, the correspondence between score functions and grid tensors is not injective~--~a score function corresponds to a single grid tensor, but a grid tensor may correspond to more than one score function.
We use the term \emph{covering} to refer to templates leading to an injective correspondence, \ie~to a situation where two score functions associated with the same grid tensor are effectively identical.
In other words, the templates $\x^{(1)}\ldots\x^{(M)}$ are covering if the value of score functions outside the exponentially large grid $\left\{X_{d_1{\ldots}d_N}:=(\x^{(d_1)},\ldots,\x^{(d_N)}):d_1{\ldots}d_N\in[M]\right\}$ is irrelevant for classification.
Some of the claims in our analysis will assume existence of covering templates (it will be stated explicitly when so).
We argue in app.~\ref{app:cover_temp} that for structured compositional data (\eg natural images), $M\in\Omega(100)$ suffices in order for this assumption to hold.

\begin{figure}
\includegraphics[width=\columnwidth]{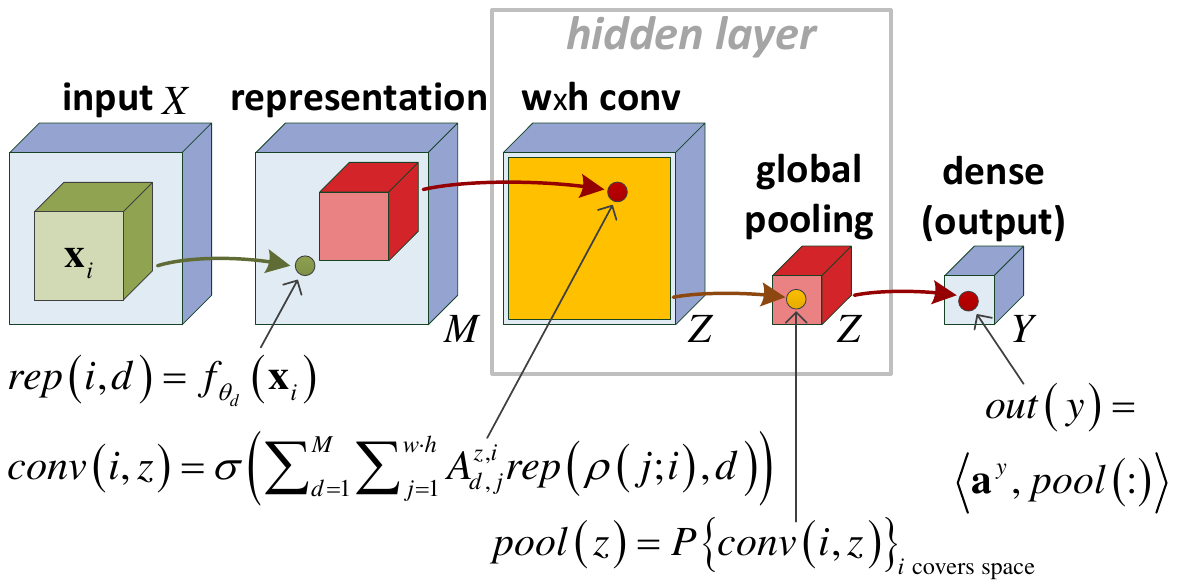}
\caption{
Shallow ConvNet with conv receptive field expanded from $1\times1$ to $w{\times}h$.
The weight vectors $\aaa^{i,z}\in\R^M$ have been replaced by matrices $A^{i,z}\in\R^{M{\times}w{\cdot}h}$, and we denote by $\rho(j;i)$ the spatial location of element $j$ in the $w{\times}h$ window revolving around $i$.
Best viewed in color.
}
\label{fig:shallow_wxh_convnet}
\end{figure}

\subsection{Matricization} \label{sec:analysis:matricization}

When analyzing grid tensors, we will often consider their arrangement as matrices.
The \emph{matricization} of a tensor $\A$, denoted $[\A]$, is its arrangement as a matrix with rows corresponding to odd modes and columns corresponding to even modes.
Specifically, if $\A\in\R^{M_1{\times\cdots\times}M_N}$, and assuming for simplicity that the order $N$ is even, the matricization $[\A]\in\R^{(M_1{\cdot}M_3{\cdot\ldots\cdot}M_{N-1})\times(M_2{\cdot}M_4{\cdot\ldots\cdot}M_N)}$ holds $\A_{d_1,\ldots,d_N}$ in row index $1+\sum_{i=1}^{\nicefrac{N}{2}}(d_{2i-1}-1)\prod_{j=i+1}^{\nicefrac{N}{2}}M_{2j-1}$ and column index $1+\sum_{i=1}^{\nicefrac{N}{2}}(d_{2i}-1)\prod_{j=i+1}^{\nicefrac{N}{2}}M_{2j}$.

The matrix analogy of the tensor product $\otimes$ (eq.~\ref{eq:tensor_prod}) is called the \emph{Kronecker product}, and is denoted by $\odot$.
For $A\in\R^{M_1{\times}M_2}$ and $B\in\R^{N_1{\times}N_2}$, $A{\odot}B$ is the matrix in $\R^{M_{1}N_{1}{\times}M_{2}N_{2}}$ holding $A_{ij}B_{kl}$ in row index $(i-1)N_1+k$ and column index $(j-1)N_2+l$.
The relation $[\A\otimes\B]=[\A]\odot[\B]$, where $\A$ and $\B$ are arbitrary tensors of even order, implies that the tensor and Kronecker products are indeed analogous, \ie they represent the same operation under tensor and matrix viewpoints, respectively.
We generalize the Kronecker product analogously to our generalization of the tensor product (eq.~\ref{eq:gen_tensor_prod}).
For an associative and commutative binary operator $g(\cdot,\cdot)$, the \emph{generalized Kronecker product} $\odotg$, is an operator that intakes matrices $A\in\R^{M_1{\times}M_2}$ and $B\in\R^{N_1{\times}N_2}$, and returns a matrix $A{\odotg}B\in\R^{M_{1}N_{1}{\times}M_{2}N_{2}}$ holding $g(A_{ij},B_{kl})$ in row index $(i-1)N_1+k$ and column index $(j-1)N_2+l$.
The relation between the tensor and Kronecker products holds for their generalized versions as well, \ie $[\A\otimesg\B]=[\A]\odotg[\B]$ for arbitrary tensors $\A,\B$ of even order.

Equipped with the matricization operator $[\cdot]$ and the generalized Kronecker product $\odotg$, we are now in a position to translate the generalized HT decomposition (eq.~\ref{eq:gen_ht_decomp}) to an expression for the matricization of a grid tensor generated by the deep ConvNet:
{\small
\bea
\phi^{1,j,\gamma} &=& \sum_{\alpha=1}^{r_0} a_\alpha^{1,j,\gamma} 
(F\aaa^{0,2j-1,\alpha}) \otimesg  (F\aaa^{0,2j,\alpha})
\label{eq:mat_gen_ht_decomp}\\
&\cdots& 
\nonumber\\
\left[\phi^{l,j,\gamma}\right] &=& \sum_{\alpha=1}^{r_{l-1}} a_\alpha^{l,j,\gamma} 
\underbrace{\left[\phi^{l-1,2j-1,\alpha}\right]}_{\text{$M^{2^{l-2}}$-by-$M^{2^{l-2}}$}} \odotg  
\underbrace{\left[\phi^{l-1,2j,\alpha}\right]}_{\text{$M^{2^{l-2}}$-by-$M^{2^{l-2}}$}} 
\nonumber\\
&\cdots& 
\nonumber\\
\left[\phi^{L-1,j,\gamma}\right] &=& \sum_{\alpha=1}^{r_{L-2}} a_\alpha^{L-1,j,\gamma} 
\underbrace{\left[\phi^{L-2,2j-1,\alpha}\right]}_{\text{$M^{\nicefrac{N}{8}}$-by-$M^{\nicefrac{N}{8}}$}} \odotg  
\underbrace{\left[\phi^{L-2,2j,\alpha}\right]}_{\text{$M^{\nicefrac{N}{8}}$-by-$M^{\nicefrac{N}{8}}$}}  
\nonumber\\ 
\left[\A\left(h_y^D\right)\right] &=& \sum_{\alpha=1}^{r_{L-1}} a_\alpha^{L,1,y} 
\underbrace{\left[\phi^{L-1,1,\alpha}\right]}_{\text{$M^{\nicefrac{N}{4}}$-by-$M^{\nicefrac{N}{4}}$}} \odotg  
\underbrace{\left[\phi^{L-1,2,\alpha}\right]}_{\text{$M^{\nicefrac{N}{4}}$-by-$M^{\nicefrac{N}{4}}$}}  
\nonumber
\eea
}
We refer to this factorization as the \emph{matricized generalized HT decomposition}.
Notice that the expression above for $\phi^{1,j,\gamma}$ is the same as in the original generalized HT decomposition, as order-2 tensors need not be matricized.

For the matricization of a grid tensor generated by the shallow ConvNet, we translate the generalized CP decomposition (eq.~\ref{eq:gen_cp_decomp}) into the \emph{matricized generalized CP decomposition}:
{\small
\bea
&&\quad\qquad\qquad\qquad\qquad\left[\A\left(h_y^S\right)\right] =
\label{eq:mat_gen_cp_decomp}\\
&&\sum_{z=1}^{Z}a_z^y\cdot \left((F\aaa^{z,1})\odotg(F\aaa^{z,3})\odotg\cdots\odotg(F\aaa^{z,N-1})\right)\odotg
\nonumber\\
&&\qquad\qquad\left((F\aaa^{z,2})\odotg(F\aaa^{z,4})\odotg\cdots\odotg(F\aaa^{z,N})\right)^\top
\nonumber
\eea
}

The matricized generalized CP and HT decompositions (eq.~\ref{eq:mat_gen_cp_decomp} and~\ref{eq:mat_gen_ht_decomp} respectively) will be used throughout our proofs to establish depth efficiency.
This is generally done by providing a lower bound on $rank[\A(h_y^D)]$~--~the rank of the deep ConvNet's matricized grid tensor, and an upper bound on $rank[\A(h_y^S)]$~--~the rank of the shallow ConvNet's matricized grid tensor.
The upper bound on $rank[\A(h_y^S)]$ will be linear in $Z$, and so requiring $\A(h_y^S)=\A(h_y^D)$, and in particular $rank[\A(h_y^S)]=rank[\A(h_y^D)]$, will give us a lower bound on $Z$.
That is to say, we obtain a lower bound on the number of hidden channels in the shallow ConvNet, that must be met in order for this network to replicate a grid tensor generated by the deep ConvNet.
Our analysis of depth efficiency is given in sec.~\ref{sec:analysis:depth_eff}.
As a prerequisite, we first head on to sec.~\ref{sec:analysis:universality} to analyze universality.

\subsection{Universality} \label{sec:analysis:universality} 

\emph{Universality} refers to the ability of a network to realize (or approximate) any function of choice when no restrictions are imposed on its size.
It is well-known that fully-connected neural networks are universal under all types of non-linear activations typically used in practice, even if the number of hidden layers is restricted to one~(\cite{Cybenko:1989fm,Hornik:1989fr,leshno1993multilayer}).
To the best of our knowledge universality has never been studied in the context of convolutional rectifier networks.
This is the purpose of the current subsection.
Specifically, we analyze the universality of our shallow and deep ConvNets (fig.~\ref{fig:shallow_convnet}, and fig.~\ref{fig:convnet} with $L=\log_{2}N$, respectively) under ReLU activation and max or average pooling.

We begin by stating a result similar to that given in~\cite{\expresstensors}, according to which convolutional arithmetic circuits are universal:

\begin{claim} \label{claim:prod_universal}
Assuming covering templates exist, with linear activation and product pooling the shallow ConvNet is universal (hence so is the deep).
\end{claim}

\begin{proof}
Let $\x^{(1)}\ldots\x^{(M)}\in\R^s$ be distinct covering templates, and $f_{\theta_1}{\ldots}f_{\theta_M}$ be representation functions for which $F$ is invertible (non-degeneracy implies that such functions exist).
With linear activation and product pooling the generalized CP decomposition (eq.~\ref{eq:gen_cp_decomp}) reduces to its standard version, which is known to be able to express any tensor when size is large enough (\eg~$Z{\geq}M^N$ suffices).
The shallow ConvNet can thus realize any grid tensor on covering templates, precisely meaning that it is universal.
As for the deep ConvNet, setting $r_0=\cdots=r_{L-1}=Z$ and $a^{l,j,\gamma}_\alpha=\indc{\alpha=\gamma}$, where $l\in[L-1]$ and $\indc{\cdot}$ is the indicator function, reduces its decomposition (eq.~\ref{eq:gen_ht_decomp}) to that of the shallow ConvNet (eq.~\ref{eq:gen_cp_decomp}).
This implies that all grid tensors realizable by the shallow ConvNet are also realizable by the deep ConvNet.
\end{proof}

Heading on to convolutional rectifier networks, the following claim tells us that max pooling leads to universality:

\begin{claim} \label{claim:max_universal}
Assuming covering templates exist, with ReLU activation and max pooling the shallow ConvNet is universal (hence so is the deep).
\end{claim}

\begin{proof}
The proof follows the same line as that of claim~\ref{claim:prod_universal}, except we cannot rely on the ability of the standard CP decomposition to realize any tensor of choice.
Instead, we need to show that the generalized CP decomposition (eq.~\ref{eq:gen_cp_decomp}) with $g(a,b)=\max\{a,b,0\}$ can realize any tensor, so long as $Z$ is large enough.
We will show that $Z{\geq}2{\cdot}M^N$ suffices.
For that, it is enough to consider an arbitrary indicator tensor, \ie a tensor holding $1$ in some entry and $0$ in all other entries, and show that it can be expressed with $Z=2$.

Let $\A$ be an indicator tensor of order $N$ and dimension $M$ in each mode, its active entry being $(d_1,\ldots,d_N)$.
Denote by $\1\in\R^M$ the vector holding $1$ in all entries, and for every $i\in[N]$, let $\bar{\e}_{d_i}\in\R^M$ be the vector holding $0$ in entry $d_i$ and $1$ elsewhere.
With the following weight settings, a generalized CP decomposition (eq.~\ref{eq:gen_cp_decomp}) with $g(a,b)=\max\{a,b,0\}$ and $Z=2$ produces $\A$, as required: 
\begin{itemize}
\item $a_1^y=1$, $a_2^y=-1$
\item $\aaa^{1,1}=\cdots=\aaa^{1,N}=\1$
\item $\forall{i\in[N]}:\aaa^{2,i}=\bar{\e}_{d_i}$
\end{itemize}
\end{proof}

At this point we encounter the first somewhat surprising result, according to which convolutional rectifier networks are not universal with average pooling:

\begin{claim} \label{claim:avg_non_universal}
With ReLU activation and average pooling, both the shallow and deep ConvNets are not universal.
\end{claim}

\begin{proof}
Let $\x^{(1)}\ldots\x^{(M)}\in\R^s$ be any templates of choice, and consider grid tensors produced by the generalized CP and HT decompositions (eq.~\ref{eq:gen_cp_decomp} and~\ref{eq:gen_ht_decomp} respectively) with $g(a,b)=\max\{a,0\}+\max\{b,0\}$ (this corresponds to \emph{sum} pooling and ReLU activation, but as stated in sec.~\ref{sec:nets2tens}, sum and average pooling are equivalent in terms of expressiveness).
We will show that such grid tensors, when arranged as matrices, necessarily have low rank.
This obviously implies that they cannot take on any value.
Moreover, since the set of low rank matrices has zero measure in the space of all matrices (see sec.~\ref{sec:analysis:prelim}), the set of values that can be taken by the grid tensors has zero measure in the space of tensors with order $N$ and dimension $M$ in each mode.

In accordance with the above, we complete our proof by showing that with $g(a,b)=\max\{a,0\}+\max\{b,0\}$, the matricized generalized CP and HT decompositions (eq.~\ref{eq:mat_gen_cp_decomp} and~\ref{eq:mat_gen_ht_decomp} respectively) give rise to low-rank matrices.
For the matricized generalized CP decomposition (eq.~\ref{eq:mat_gen_cp_decomp}), corresponding to the shallow ConvNet, we have with $g(a,b)=\max\{a,0\}+\max\{b,0\}$:
$$\left[\A\left(h_y^S\right)\right]=\vv\1^\top+\1\uu^\top$$
where $\1$ is the vector in $\R^{M^{N/2}}$ holding $1$ in all entries, and $\vv,\uu\in\R^{M^{N/2}}$ are defined as follows:
\beas
\vv&:=&\sum_{z=1}^{Z}a_z^y\cdot\max\left\{(F\aaa^{z,1})\odotg\cdots\odotg(F\aaa^{z,N-1}),0\right\}\\
\uu&:=&\sum_{z=1}^{Z}a_z^y\cdot\max\left\{(F\aaa^{z,2})\odotg\cdots\odotg(F\aaa^{z,N}),0\right\}
\eeas
Obviously the matrix $\left[\A\left(h_y^S\right)\right]\in\R^{M^{N/2}{\times}M^{N/2}}$ has rank $2$ or less.

Turning to the matricized generalized HT decomposition (eq.~\ref{eq:mat_gen_ht_decomp}), which corresponds to the deep ConvNet, we have with $g(a,b)=\max\{a,0\}+\max\{b,0\}$:
$$\left[\A\left(h_y^D\right)\right]=V{\odot}O+O{\odot}U$$
where $\odot$ is the standard Kronecker product (see definition in sec.~\ref{sec:analysis:matricization}), $O\in\R^{M^{N/4}{\times}M^{N/4}}$ is a matrix holding $1$ in all entries, and the matrices $V,U\in\R^{M^{N/4}{\times}M^{N/4}}$ are given by:
\beas
V&:=&\sum_{\alpha=1}^{r_{L-1}}a_\alpha^{L,1,y}\max\left\{\left[\phi^{L-1,1,\alpha}\right],0\right\} \\
U&:=&\sum_{\alpha=1}^{r_{L-1}}a_\alpha^{L,1,y}\max\left\{\left[\phi^{L-1,2,\alpha}\right],0\right\}
\eeas
The rank of $O$ is obviously $1$, and since the Kronecker product multiplies ranks, \ie $rank(A{\odot}B)=rank(A){\cdot}rank(B)$ for any matrices $A$ and $B$, we have that the rank of $\left[\A\left(h_y^D\right)\right]\in\R^{M^{N/2}{\times}M^{N/2}}$ is at most $2{\cdot}M^{N/4}$.
In particular, $\left[\A\left(h_y^D\right)\right]$ cannot have full rank.
\end{proof}

One may wonder if perhaps the non-universality of ReLU activation and average pooling is merely an artifact of the conv operator in our ConvNets having $1\times1$ receptive field.
Apparently, as the following claim shows, expanding the receptive field does not remedy the situation, and indeed non-universality is an inherent property of convolutional rectifier networks with average pooling:

\begin{claim} \label{claim:avg_shallow_wxh_non_universal}
Consider the network illustrated in fig.~\ref{fig:shallow_wxh_convnet}, obtained by expanding the conv receptive field in the shallow ConvNet from $1\times1$ to $w{\times}h$, where $w{\cdot}h<\nicefrac{N}{2}+1-\log_M{N}$ (conv windows cover less than half the feature maps that precede them).
Such a network, when equipped with ReLU activation and average pooling, is not universal.
\end{claim}

\begin{proof}
Compare the original shallow ConvNet (fig.~\ref{fig:shallow_convnet}) to the shallow ConvNet with expanded receptive field that we consider in this claim (fig.~\ref{fig:shallow_wxh_convnet}).
The original shallow ConvNet has $1\times1$ receptive field, with conv entry in location $i\in[N]$ and channel $z\in[Z]$ assigned through a cross-channel linear combination of the representation entries in the same location, the combination weights being $\aaa^{z,i}\in\R^M$.
In the shallow ConvNet with receptive field expanded to $w{\times}h$, linear combinations span multiple locations.
In particular, conv entry in location $i$ and channel $z$ is now assigned through a linear combination of the representation entries at all channels that lie inside a spatial window revolving around $i$.
We denote by $\{\rho(j;i)\}_{j\in[w{\cdot}h]}$ the locations comprised by this window.
More specifically, $\rho(j;i)$ is the $j$'th location in the window, and the linear weights that correspond to it are held in the $j$'th column of the weight matrix $A^{z,i}\in\R^{M{\times}w{\cdot}h}$.
We assume for simplicity that conv windows stepping out of bounds encounter zero padding
\footnote{
Modifying our proof to account for different padding schemes (such as duplication or no padding at all) is trivial~--~we choose to work with zero padding merely for notational convenience.
}, 
and adhere to the convention under which indexing the row of a matrix with $d_{\rho(j;i)}$ produces zero when location $j$ of window $i$ steps out of bounds.

We are interested in the case of ReLU activation ($\sigma(z)=\max\{0,z\}$) and average pooling ($P\{c_j\}=\mean\{c_j\}$).
Under this setting, for any selected templates $\x^{(1)}\ldots\x^{(M)}\in\R^s$, the grid tensor of $h_y^{S(w{\times}h)}$~--~network's $y$'th score function, is given by:
$$\A(h_y^{S(w{\times}h)})_{d_1,\ldots,d_N}=\sum_{i=1}^{N}\B^{i}_{d_{\rho(1;i)},\ldots,d_{\rho(w{\cdot}h;i)}}$$
where for every $i\in[N]$, $\B^i$ is a tensor of order $w{\cdot}h$ and dimension $M$ in each mode, defined by:
$$\B^i_{c_1,\ldots,c_{w{\cdot}h}}=\sum_{z=1}^{Z}\frac{a_z^y}{N}\max\left\{\sum_{j=1}^{w{\cdot}h}(F A^{z,i})_{c_j,j},0\right\}$$
Let $\OO$ be a tensor of order $N-w{\cdot}h$ and dimension $M$ in each mode, holding $1$ in all entries.
We may write:
\be
\A(h_y^{S(w{\times}h)})=\sum_{i=1}^{N}p_i(\B^{i}\otimes\OO)
\label{eq:shallow_wxh_convnet_grid_tensor}
\ee
where for every $i\in[N]$, $p_i(\cdot)$ is an appropriately chosen operator that permutes the modes of an order-$N$ tensor.

We now make use of some known facts related to tensor rank (see sec.~\ref{sec:analysis:prelim}), in order to show that eq.~\ref{eq:shallow_wxh_convnet_grid_tensor} is not universal, \ie that there are many tensors which cannot be realized by $\A(h_y^{S(w{\times}h)})$.
Being tensors of order $w{\cdot}h$ and dimension $M$ in each mode, the ranks of $\B^1\ldots\B^N$ are bounded above by $M^{w{\cdot}h-1}$.
Since $\OO$ is an all-$1$ tensor, and since permuting modes does not alter rank, we have: $rank(p_i(\B^{i}\otimes\OO)){\leq}M^{w{\cdot}h-1}~\forall{i\in[N]}$.
Finally, from sub-additivity of the rank we get: $rank(\A(h_y^{S(w{\times}h)})){\leq}N{\cdot}M^{w{\cdot}h-1}$.
Now, we know by assumption that $w{\cdot}h<\nicefrac{N}{2}+1-\log_M{N}$, and this implies: $rank(\A(h_y^{S(w{\times}h)}))<M^{\nicefrac{N}{2}}$.
Since there exist tensors of order $N$ and dimension $M$ in each mode having rank at least $M^{\nicefrac{N}{2}}$ (actually only a negligible set of tensors do not meet this), eq.~\ref{eq:shallow_wxh_convnet_grid_tensor} is indeed not universal.
That is to say, the shallow ConvNet with conv receptive field expanded to $w{\times}h$ (fig.~\ref{fig:shallow_wxh_convnet}) cannot realize all grid tensors on the templates $\x^{(1)}\ldots\x^{(M)}$.
\end{proof}

We conclude this subsection by noting that the non-universality result in claim~\ref{claim:avg_shallow_wxh_non_universal} does \emph{not} contradict the known universality of shallow (single hidden layer) fully-connected neural networks.
Indeed, a shallow fully-connected network corresponds to the ConvNet illustrated in fig.~\ref{fig:shallow_wxh_convnet}, with conv receptive field covering the entire spatial extent ($w{\cdot}h=N$), thereby effectively removing the pooling operator (assuming the latter realizes the identity on singletons).
In claim~\ref{claim:shallow_mlp_universal} below we show that such a network, when equipped with ReLU activation, is universal.
On the other hand, in claim~\ref{claim:avg_shallow_wxh_non_universal} we assumed that the ConvNet's receptive field covers less than half the spatial extent ($w{\cdot}h<\nicefrac{N}{2}+1-\log_M{N}$), and have shown that with ReLU activation and average pooling, this leads to non-universality.
Loosely speaking, our findings imply that for networks with ReLU activation, which are known to be universal when fully-connected, introducing locality disrupts universality with average pooling (and maintains it with max pooling).

\begin{claim} \label{claim:shallow_mlp_universal}
Assume that there exist covering templates $\x^{(1)}\ldots\x^{(M)}$, and corresponding representation functions $f_{\theta_1}{\ldots}f_{\theta_M}$ leading to a matrix $F$ (eq.~\ref{eq:F}) that has non-recurring rows and a constant non-zero column
\footnote{
The assumption that such representation functions exist differs from our usual non-degeneracy assumption.
The latter requires $F$ to be non-singular, whereas here we pose the weaker requirement of $F$ having non-recurring rows.
On the other hand, here we also demand that $F$ have a constant non-zero column, \ie~that there be a representation function $f_{\theta_d}$ such that $f_{\theta_d}(\x^{(1)})=\cdots=f_{\theta_d}(\x^{(M)})=c\neq0$.
In claim~\ref{claim:F_inv_f4x} we showed that standard neurons meet the non-degeneracy assumption.
A slight modification to its proof shows that they also meet the assumption made here.
Namely, if we modify the constructions for the cases of ReLU activation and sigmoidal activation by setting $f_{\theta_1}(\x)=\sigma(\0^\top\x+1)$ and $f_{\theta_1}(\x)=\sigma(\0^\top\x+\alpha)$ respectively, we get matrices $F$ that are not only non-singular, but also have a constant non-zero column.
}.
Consider the fully-connected network illustrated in fig.~\ref{fig:shallow_mlp}, obtained by expanding the conv receptive field in the shallow ConvNet to cover the entire spatial extent.
Such a network, when equipped with ReLU activation, is universal.
\end{claim}

\begin{figure}
\includegraphics[width=\columnwidth]{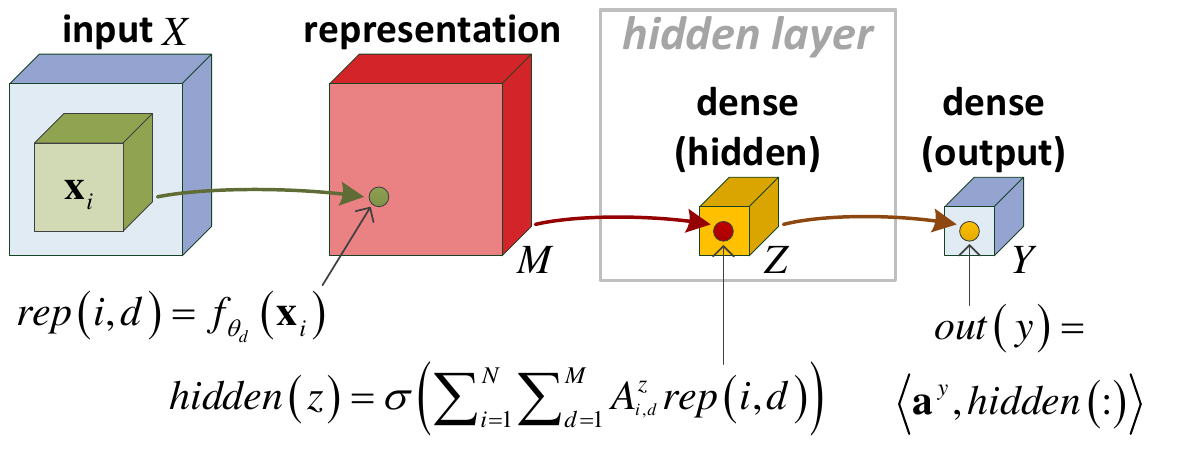}
\caption{
Shallow fully-connected network obtained by expanding the conv receptive field in the shallow ConvNet to cover the entire spatial extent.
The hidden layer consists of a $Z$-channel dense linear operator weighted by $\{A^z\in\R^{N{\times}M}\}_{z\in[Z]}$, and followed by point-wise activation $\sigma(\cdot)$.
The resulting $Z$-dimensional vector is mapped to $Y$ network outputs through a dense linear operator weighted by $\{\aaa^y\in\R^Z\}_{y\in[Y]}$.
Best viewed in color.
}
\label{fig:shallow_mlp}
\end{figure}

\begin{proof}
Let $h_y^{S(fc)}$ be the $y$'th score function of our shallow fully-connected network (fig.~\ref{fig:shallow_mlp}) when equipped with ReLU activation ($\sigma(z)=\max\{0,z\}$).
We would like to show that $\A(h_y^{S(fc)})$~--~the grid tensor of $h_y^{S(fc)}$ \wrt the covering templates $\x^{(1)}\ldots\x^{(M)}$, may take on any value when hidden and output weights ($\{A^z\}_{z\in[Z]}$ and $\aaa^y$ respectively) are chosen appropriately.

For any $d_1{\ldots}d_N\in[M]$, define the following matrix:
$$F^{(d_1{\ldots}d_N)}:=
\begin{bmatrix}
f_{\theta_1}(\x^{(d_1)}) & \cdots & f_{\theta_M}(\x^{(d_1)}) \\
\vdots & \ddots & \vdots \\
f_{\theta_1}(\x^{(d_N)}) & \cdots & f_{\theta_M}(\x^{(d_N)})
\end{bmatrix}
\in\R^{N{\times}M}$$
In words, $F^{(d_1{\ldots}d_N)}$ is the matrix obtained by taking rows $d_1{\ldots}d_N$ from $F$ (recurrence allowed), and stacking them one on top of the other.
It holds that:
$$\A(h_y^{S(fc)})_{d_1{\ldots}d_N}=\sum_{z=1}^{Z}a_{z}^{y}\max\left\{0,\inprod{F^{(d_1{\ldots}d_N)}}{A^z}\right\}$$
where $\inprod{\cdot}{\cdot}$ stands for the inner-product operator, \ie $\inprod{F^{(d_1{\ldots}d_N)}}{A^z}:=\sum_{i=1}^{N}\sum_{d=1}^{M}F_{i,d}^{(d_1{\ldots}d_N)}A_{i,d}^{z}$.

By assumption $F$ has a constant non-zero column.
This implies that there exist $j\in[M],c\neq0$ such that for any $d_1{\ldots}d_N\in[M]$, all entries in column $j$ of $F^{(d_1{\ldots}d_N)}$ are equal to~$c$.
For every $d_1{\ldots}d_N\in[M]$ and $z\in[Z]$, denote by $\tilde{F}^{(d_1{\ldots}d_N)}$ and $\tilde{A}^z$ the matrices obtained by removing the $j$'th column from $F^{(d_1{\ldots}d_N)}$ and $A^z$ respectively.
Defining $\bb\in\R^Z$ to be the vector whose $z$'th entry is given by $b_z=c\cdot\sum_{i=1}^{N}A_{i,j}^z$, we may write:
$$\A(h_y^{S(fc)})_{d_1{\ldots}d_N}=\sum_{z=1}^{Z}a_{z}^{y}\max\left\{0,\inprod{\tilde{F}^{(d_1{\ldots}d_N)}}{\tilde{A}^z}+b_z\right\}$$
noting that for every $z\in[Z]$, $\tilde{A}^z$ and $b_z$ may take on any values with proper choice of $A^z$.
Since by assumption $F$ has non-recurring rows, and since all rows hold the same value ($c$) in their $j$'th entry, we have that $\tilde{F}^{(d_1{\ldots}d_N)}\neq\tilde{F}^{(d'_1{\ldots}d'_N)}$ for $(d_1{\ldots}d_N)\neq(d'_1{\ldots}d'_N)$.
An application of lemma~\ref{lemma:piecewise_affine_universal} now shows that when $Z{\geq}M^N$, any value for the grid tensor $\A(h_y^{S(fc)})$ may be realized with proper assignment of $\{\tilde{A}^z\}_{z\in[Z]}$, $\bb$ and $\aaa^y$.
Since $\{\tilde{A}^z\}_{z\in[Z]}$ and $\bb$ may be set arbitrarily through $\{A^z\}_{z\in[Z]}$, we get that with proper choice of hidden and output weights ($\{A^z\}_{z\in[Z]}$ and $\aaa^y$ respectively), the grid tensor of our network \wrt the covering templates may take on any value, precisely meaning that universality holds.
\end{proof}

\begin{lemma} \label{lemma:piecewise_affine_universal}
Let $\vv_1\ldots\vv_k\in\R^D$ be distinct vectors ($\vv_i\neq\vv_j$ for $i{\neq}j$), and $c_1{\ldots}c_k\in\R$ be any scalars.
Then, there exist $\w_1\ldots\w_k\in\R^D$, $\bb\in\R^k$ and $\aaa\in\R^k$ such that $\forall{i\in[k]}$:
\be
\sum_{j=1}^{k}a_j\max\{0,\w_j^\top\vv_i+b_j\}=c_i
\label{eq:piecewise_affine_universal}
\ee
\end{lemma}

\begin{proof}
As shown in the proof of claim~\ref{claim:F_inv_f4x}, for distinct $\vv_1\ldots\vv_k\in\R^D$ there exists a vector $\uu\in\R^D$ such that $\uu^\top\vv_i\neq\uu^\top\vv_j$ for all $1{\leq}i<j{\leq}k$.
We assume without loss of generality that $\uu^\top\vv_1<\ldots<\uu^\top\vv_k$, and set $\w_1{\ldots}\w_k$, $\bb$ and $\aaa$ as follows:
\begin{itemize}
\item $\w_1=\cdots=\w_k=\uu$
\item $b_1=-\uu^\top\vv_1+1$
\item $b_j=-\uu^\top\vv_{j-1}$ for $j=2{\ldots}k$
\item $a_1=c_1$
\item $a_j=\frac{c_j-c_{j-1}}{\uu^\top\vv_j-\uu^\top\vv_{j-1}}-\sum_{t=1}^{j-1}a_t$ for $j=2{\ldots}k$
\end{itemize}
To complete the proof, we show below that this assignment meets the condition in eq.~\ref{eq:piecewise_affine_universal} for $i=1{\ldots}k$.

\medskip

The fact that:
$$\w_j^\top\vv_1+b_j=\left\{
	\begin{array}{ll}
		\uu^\top\vv_1-\uu^\top\vv_1+1=1  & ,j=1 \\
		\uu^\top\vv_1-\uu^\top\vv_{j-1}\leq0 & ,2{\leq}j{\leq}k
	\end{array}
\right.$$
implies that the condition in eq.~\ref{eq:piecewise_affine_universal} indeed holds for $i=1$:
$$\sum_{j=1}^{k}a_j\max\{0,\w_j^\top\vv_1+b_j\}=a_1{\cdot}1+\sum_{j=1}^{k}a_j{\cdot}0=a_1=c_1$$
For $i>1$ we have:
$$\w_j^\top\vv_i+b_j=\left\{
	\begin{array}{ll}
		\uu^\top\vv_i-\uu^\top\vv_1+1>0     & ,j=1 \\
		\uu^\top\vv_i-\uu^\top\vv_{j-1}>0    & ,2{\leq}j{\leq}i \\
		\uu^\top\vv_i-\uu^\top\vv_{j-1}\leq0 & ,i<j{\leq}k
	\end{array}
\right.$$
which implies:
\beas
&\sum_{j=1}^{k}a_j\max\{0,\w_j^\top\vv_i+b_j\}=& \\
&a_1(\uu^\top\vv_i-\uu^\top\vv_1+1)+\sum_{j=2}^{i}a_j(\uu^\top\vv_i-\uu^\top\vv_{j-1})&
\eeas
Comparing this to the same expression with $i$ replaced by $i-1$ we obtain:
\beas
&\sum_{j=1}^{k}a_j\max\{0,\w_j^\top\vv_i+b_j\}=& \\
&\sum_{j=1}^{k}a_j\max\{0,\w_j^\top\vv_{i-1}+b_j\}+& \\
&(\uu^\top\vv_i-\uu^\top\vv_{i-1})\sum_{j=1}^{i}a_j&
\eeas
Now, if we follow an inductive argument and assume that the condition in eq.~\ref{eq:piecewise_affine_universal} holds for $i-1$, \ie that $\sum_{j=1}^{k}a_j\max\{0,\w_j^\top\vv_{i-1}+b_j\}=c_{i-1}$, we get:
\beas
&\sum_{j=1}^{k}a_j\max\{0,\w_j^\top\vv_i+b_j\}=& \\
&c_{i-1}+(\uu^\top\vv_i-\uu^\top\vv_{i-1})\sum_{j=1}^{i}a_j&
\eeas
Plugging in the definition $a_i=\frac{c_i-c_{i-1}}{\uu^\top\vv_i-\uu^\top\vv_{i-1}}-\sum_{j=1}^{i-1}a_j$ gives:
\beas
&\sum_{j=1}^{k}a_j\max\{0,\w_j^\top\vv_i+b_j\}=& \\
&c_{i-1}+(\uu^\top\vv_i-\uu^\top\vv_{i-1})\frac{c_i-c_{i-1}}{\uu^\top\vv_i-\uu^\top\vv_{i-1}}=c_i&
\eeas
Thus the condition in eq.~\ref{eq:piecewise_affine_universal} holds for $i$ as well.
We have therefore shown by induction that our assignment of $\w_1{\ldots}\w_k$, $\bb$ and $\aaa$ meets the lemma's requirement.
\end{proof}

\subsection{Depth Efficiency} \label{sec:analysis:depth_eff}

The driving force behind deep learning is the expressive power that comes with depth.
It is generally believed that deep networks with non-linear layers efficiently express functions that cannot be efficiently expressed by shallow networks, \ie that would require the latter to have super-polynomial size.
We refer to such scenario as \emph{depth efficiency}.
Being concerned with the minimal size required by a shallow network in order to realize (or approximate) a given function, the question of depth efficiency implicitly assumes universality, \ie that there exists some (possibly exponential) size with which the shallow network is capable of expressing the target function.
\footnote{
While technically it is possible to consider depth efficiency with a non-universal shallow network, in the majority of the cases, particularly in our framework, the shallow network would simply not be able to express a function generated by a deep network, no matter how large we allow it to be.
Arguably, this provides little insight into the complexity of functions brought forth by depth.
}

To the best of our knowledge, at the time of this writing the only work to formally analyze depth efficiency in the context of ConvNets is~\cite{\expresstensors}.
This work focused on convolutional arithmetic circuits, showing that with such networks depth efficiency is \emph{complete}, \ie besides a negligible set, all functions realizable by a deep network enjoy depth efficiency.
We frame this result in our setup:

\begin{claim}[adaptation of theorem~1 in~\cite{\expresstensors}] \label{claim:prod_depth_eff_complete}
Let $f_{\theta_1}{\ldots}f_{\theta_M}$ be any set of linearly independent representation functions for a deep ConvNet (fig.~\ref{fig:convnet} with $L=\log_{2}N$) with linear activation and product pooling.
Suppose we randomize the linear weights ($\aaa^{l,j,\gamma}$) of the network by some continuous distribution.
Then, with probability~1, we obtain score functions that cannot be realized by a shallow ConvNet (fig.~\ref{fig:shallow_convnet}) with linear activation and product pooling if the number of hidden channels in the latter ($Z$) is less than $\min\{r_0,M\}^{\nicefrac{N}{2}}$.
\end{claim}

\begin{proof}
Let $\x^{(1)}\ldots\x^{(M)}\in\R^s$ be templates such that $F$ is invertible (existence follows from claim~\ref{claim:F_inv_x4f}).
The deep network generates grid tensors on $\x^{(1)}\ldots\x^{(M)}$ through the standard HT decomposition (eq.~\ref{eq:gen_ht_decomp} with $g(a,b)=a{\cdot}b$).
The proof of theorem~1 in~\cite{\expresstensors} shows that when arranged as matrices, such tensors have rank at least $\min\{r_0,M\}^{\nicefrac{N}{2}}$ almost always, \ie for all weight ($\aaa^{l,j,\gamma}$) settings but a set of (Lebesgue) measure zero.
On the other hand, the shallow network generates grid tensors on $\x^{(1)}\ldots\x^{(M)}$ through the standard CP decomposition (eq.~\ref{eq:gen_cp_decomp} with $g(a,b)=a{\cdot}b$), possibly with a different matrix $F$ (representation functions need not be the same).
Such tensors, when arranged as matrices, are shown in the proof of theorem~1 in~\cite{\expresstensors} to have rank at most $Z$.
Therefore, for them to realize the grid tensors generated by the deep network, we almost always must have $Z\geq\min\{r_0,M\}^{\nicefrac{N}{2}}$.
\end{proof}

We now turn to convolutional rectifier networks, for which depth efficiency has yet to be analyzed.
In sec.~\ref{sec:analysis:universality} we saw that convolutional rectifier networks are universal with max pooling, and non-universal with average pooling.
Since depth efficiency is only applicable to universal architectures, we focus on the former setting.
The following claim establishes existence of depth efficiency for ConvNets with ReLU activation and max pooling:

\begin{claim} \label{claim:max_depth_eff_exist}
There exist weight settings for a deep ConvNet with ReLU activation and max pooling, giving rise to score functions that cannot be realized by a shallow ConvNet with ReLU activation and max pooling if the number of hidden channels in the latter ($Z$) is less than $\min\{r_0,M\}^{\nicefrac{N}{2}}\cdot\frac{2}{M{\cdot}N}$.
\end{claim}

\begin{proof}
The proof traverses along the following path.
Letting $\x^{(1)}\ldots\x^{(M)}\in\R^s$ be any distinct templates, we show that when arranged as matrices, grid tensors on $\x^{(1)}\ldots\x^{(M)}$ generated by the shallow network have rank at most $Z{\cdot}\frac{M{\cdot}N}{2}$.
Then, defining $f_{\theta_1}{\ldots}f_{\theta_M}$ to be representation functions for the deep network giving rise to an invertible $F$ (non-degeneracy implies that such functions exist), we show explicit linear weight ($\aaa^{l,j,\gamma}$) settings under which the grid tensors on $\x^{(1)}\ldots\x^{(M)}$ generated by the deep network, when arranged as matrices, have rank at least $\min\{r_0,M\}^{\nicefrac{N}{2}}$.

\medskip

In light of the above, the proof boils down to showing that with $g(a,b)=\max\{a,b,0\}$:
\begin{itemize}
\item The matricized generalized CP decomposition (eq.~\ref{eq:mat_gen_cp_decomp}) produces matrices with rank at most $Z{\cdot}\frac{M{\cdot}N}{2}$.
\item For an invertible $F$, there exists a weight ($\aaa^{l,j,\gamma}$) setting under which the matricized generalized HT decomposition (eq.~\ref{eq:mat_gen_ht_decomp}) produces a matrix with rank at least $\min\{r_0,M\}^{\nicefrac{N}{2}}$.
\end{itemize}

We begin with the first point, showing that for every $\vv_{1},\ldots,\vv_{\nicefrac{N}{2}}\in\R^M$ and $\uu_{1},\ldots,\uu_{\nicefrac{N}{2}}\in\R^M$:
\be
rank\left(\vv_{1}\odotg\cdots\odotg\vv_{\frac{N}{2}}\right)\odotg\left(\uu_{1}\odotg\cdots\odotg\uu_{\frac{N}{2}}\right)^\top
{\leq}\frac{M{\cdot}N}{2}
\label{eq:max_shallow_mat_rank_bound}
\ee
This would imply that every summand in the matricized generalized CP decomposition (eq.~\ref{eq:mat_gen_cp_decomp}) has rank at most $\frac{M{\cdot}N}{2}$, and the desired result readily follows.
To prove eq.~\ref{eq:max_shallow_mat_rank_bound}, note that each of the vectors $\bar{\vv}:=\vv_{1}\odotg\cdots\odotg\vv_{\frac{N}{2}}$ and $\bar{\uu}:=\uu_{1}\odotg\cdots\odotg\uu_{\frac{N}{2}}$ are of dimension $M^{\nicefrac{N}{2}}$, but have only up to $\frac{M{\cdot}N}{2}$ unique values.
Let $\delta_{\vv},\delta_{\uu}:[M^{N/2}]\to[M^{N/2}]$ be permutations that arrange the entries of $\bar{\vv}$ and $\bar{\uu}$ in descending order.
Permuting the rows of the matrix $\bar{\vv}\odotg\bar{\uu}^\top$ via $\delta_{\vv}$, and the columns via $\delta_{\uu}$, obviously does not change its rank.
On the other hand, we get a $M^{N/2}{\times}M^{N/2}$ matrix with a $\frac{M{\cdot}N}{2}\times\frac{M{\cdot}N}{2}$ block structure, each block being constant (\ie all entries of a block hold the same value).
This implies that the rank of $\bar{\vv}\odotg\bar{\uu}^\top$ is at most $\frac{M{\cdot}N}{2}$, which is what we set out to prove.

Moving on to the matricized generalized HT decomposition (eq.~\ref{eq:mat_gen_ht_decomp}), for an invertible $F$ we define the following weight setting ($\0$ and $\1$ here denote the all-$0$ and all-$1$ vectors, respectively):
\begin{itemize}
\item $\aaa^{0,j,\gamma}=\left\{
	\begin{array}{ll}
		F^{-1}\bar{\e}_{\gamma}  & ,\gamma{\leq}M \\
		\0                                   & ,\gamma>M
	\end{array}
\right.$
, where $\bar{\e}_{\gamma}\in\R^M$ is defined to be the vector holding $0$ in entry $\gamma$ and $1$ in all other entries.
\item $\aaa^{l,j,\gamma}=\left\{
	\begin{array}{ll}
		\1  & ,\gamma=1~,~l\in[L-1] \\
		\0  & ,\gamma>1~,~l\in[L-1]
	\end{array}
\right.$
\item $\aaa^{L,1,y}=\1$
\end{itemize}
Under this setting, the produced matrix $\left[\A\left(h_y^D\right)\right]$ holds $\min\{r_0,M\}$ everywhere besides $\min\{r_0,M\}^{\nicefrac{N}{2}}$ entries on its diagonal, where it holds $\min\{r_0,M\}-1$. 
The rank of this matrix is at least $\min\{r_0,M\}^{\nicefrac{N}{2}}$.
\end{proof}

Nearly all results in the literature that relate to depth efficiency merely show its existence, and claim~\ref{claim:max_depth_eff_exist} is no different in that respect.
From a practical perspective, the implications of such results are slight, as a-priori, it may be that only a small fraction of the functions realizable by a deep network enjoy depth efficiency, and for all the rest shallow networks suffice.
In sec.~\ref{sec:analysis:depth_eff:depth_eff_incidence} we extend claim~\ref{claim:max_depth_eff_exist}, arguing that with ReLU activation and max pooling, depth efficiency becomes more and more prevalent as the number of hidden channels in the deep ConvNet grows.
However, no matter how large the deep ConvNet is, with ReLU activation and max pooling depth efficiency is {\it never complete}~--~there is always positive measure to the set of weight configurations that lead the deep ConvNet to generate score functions efficiently realizable by the shallow ConvNet:

\begin{claim} \label{claim:max_depth_eff_incomplete}
Suppose we randomize the weights of a deep ConvNet with ReLU activation and max pooling by some continuous distribution with non-vanishing continuous probability density function.
Then, assuming covering templates exist, with positive probability, we obtain score functions that can be realized by a shallow ConvNet with ReLU activation and max pooling having only a single hidden channel ($Z=1$).
\end{claim}

\begin{proof}
Let $\x^{(1)}\ldots\x^{(M)}\in\R^s$ be covering templates, and $f_{\theta_1}{\ldots}f_{\theta_M}$ be representation functions for the deep network under which $F$ is invertible (non-degeneracy implies that such functions exist).
We will show that there exists a linear weight ($\aaa^{l,j,\gamma}$) setting for the deep network with which it generates a grid tensor that is realizable by a shallow network with a single hidden channel ($Z=1$).
Moreover, we show that when the representation parameters ($\theta_d$) and linear weights ($\aaa^{l,j,\gamma}$) are subject to small perturbations, the deep network's grid tensor can still be realized by a shallow network with a single hidden channel.
Since templates are covering grid tensors fully define score functions.
This, along with the fact that open sets in Lebesgue measure spaces always have positive measure (see sec.~\ref{sec:analysis:prelim}), imply that there is positive measure to the set of weight configurations leading the deep network to generate score functions realizable by a shallow network with $Z=1$.
Translating the latter statement from measure theoretical to probabilistic terms readily proves the result we seek after.

\medskip

In light of the above, the proof boils down to the following claim, framed in terms of our generalized tensor decompositions.
Fixing $g(a,b)=\max\{a,b,0\}$, per arbitrary invertible $F$ there exists a weight ($\aaa^{l,j,\gamma}$) setting for the generalized HT decomposition (eq.~\ref{eq:gen_ht_decomp}), such that the produced tensor may be realized by the generalized CP decomposition (eq.~\ref{eq:gen_cp_decomp}) with $Z=1$, and this holds even if the weights $\aaa^{l,j,\gamma}$ and matrix $F$ are subject to small perturbations
\footnote{
Recall that by assumption representation functions are continuous \wrt their parameters ($f_\theta(\x)$ is continuous \wrt $\theta$), and so small perturbations on representation parameters ($\theta_d$) translate into small perturbations on the matrix $F$ (eq.~\ref{eq:F}).
}.

We will now show that the following weight setting meets our requirement ($\0$ and $\1$ here denote the all-$0$ and all-$1$ vectors, respectively):
\begin{itemize}
\item $\aaa^{0,j,\gamma}=\left\{
	\begin{array}{ll}
		F^{-1}\1  & ,j~\text{odd} \\
		\0           & ,j~\text{even}
	\end{array}
\right.$
\item $\aaa^{l,j,\gamma}=\left\{
	\begin{array}{ll}
		\1  & ,j~\text{odd}~~,~l\in[L-1] \\
		\0  & ,j~\text{even}~,~l\in[L-1]
	\end{array}
\right.$
\item $\aaa^{L,1,y}=\1$
\end{itemize}
Let $\E^F$ be an additive noise matrix applied to $F$, and $\{\epsbf^{l,j,\gamma}\}_{l,j,\gamma}$ be additive noise vectors applied to $\{\aaa^{l,j,\gamma}\}_{l,j,\gamma}$.
We use the notation $\oo(\epsilon)$ to refer to vectors that tend to $\0$ as $\E^F\to0$ and $\epsbf^{l,j,\gamma}\to\0$, with the dimension of a vector to be understood by context.
Plugging in the noisy variables into the generalized HT decomposition (eq.~\ref{eq:gen_ht_decomp}), we get for every $j\in[N/2]$ and $\alpha\in[r_0]$:
\beas
&((F+\E^F)(\aaa^{0,2j-1,\alpha}+\epsbf^{0,2j-1,\alpha}))& \\
&\otimesg((F+\E^F)(\aaa^{0,2j,\alpha}+\epsbf^{0,2j,\alpha}))& \\
&=((F+\E^F)(F^{-1}\1+\epsbf^{0,2j-1,\alpha}))& \\
&\otimesg((F+\E^F)(\0+\epsbf^{0,2j,\alpha}))& \\
&=(\1+\oo(\epsilon))\otimesg\oo(\epsilon)&
\eeas
If the applied noise ($\E^F,\epsbf^{l,j,\gamma}$) is small enough this is equal to $(\1+\oo(\epsilon))\otimes\1$ (recall that $\otimes$ stands for the \emph{standard} tensor product), and we in turn get for every $j\in[N/4]$ and $\gamma\in[r_1]$:
\beas
&\phi^{1,2j-1,\gamma}\otimesg\phi^{1,2j,\gamma}& \\
&=\left(\sum_{\alpha=1}^{r_0} a_\alpha^{1,2j-1,\gamma} (\1+\oo(\epsilon))\otimes\1\right)& \\
&\otimesg\left(\sum_{\alpha=1}^{r_0} a_\alpha^{1,2j,\gamma} (\1+\oo(\epsilon))\otimes\1\right)& \\
&=\left(\sum_{\alpha=1}^{r_0}(1+\epsilon_\alpha^{1,2j-1,\gamma}) (\1+\oo(\epsilon))\otimes\1\right)& \\
&\otimesg\left(\sum_{\alpha=1}^{r_0} \epsilon_\alpha^{1,2j,\gamma} (\1+\oo(\epsilon))\otimes\1\right)& \\
&=\left((r_0\1+\oo(\epsilon))\otimes\1\right)\otimesg\left(\oo(\epsilon)\otimes\1\right)
\eeas
With the applied noise ($\E^F,\epsbf^{l,j,\gamma}$) small enough this becomes $(r_0\1+\oo(\epsilon)\otimes\1\otimes\1\otimes\1$.
Continuing in this fashion over the levels of the decomposition, we get that with small enough noise, for every $l\in[L-1]$, $j\in[N/2^{l+1}]$ and $\gamma\in[r_l]$:
$$\phi^{l,2j-1,\gamma}\otimesg\phi^{l,2j,\gamma}=
\left(\prod\nolimits_{l'=0}^{l-1}r_{l'}\cdot\1+\oo(\epsilon)\right)\otimes\left(\otimes_{i=1}^{2^{l+1}-1}\1\right)$$
where $\otimes_{i=1}^{2^{l+1}-1}\1$ stands for the tensor product of the vector $\1$ with itself $2^{l+1}-1$ times.
We readily conclude from this that with small enough noise, the tensor produced by the decomposition may be written as follows:
\be
\A\left(h_y^D\right)=\left(\prod\nolimits_{l=0}^{L-1}r_{l}\cdot\1+\oo(\epsilon)\right)\otimes\left(\otimes_{i=1}^{N-1}\1\right)
\label{eq:max_deep_trivial_tensor}
\ee

To finish our proof, it remains to show that a tensor as in eq.~\ref{eq:max_deep_trivial_tensor} may be realized by the generalized CP decomposition (eq.~\ref{eq:gen_cp_decomp}) with $Z=1$ (and $g(a,b)=\max\{a,b,0\}$).
Indeed, we may assume that the latter's $F$, which we denote by $\tilde{F}$ to distinguish from the matrix in the generalized HT decomposition (eq.~\ref{eq:gen_ht_decomp}), is invertible (non-degeneracy ensures that this may be achieved with proper choice of representation functions for the shallow ConvNet).
Setting the weights of the generalized CP decomposition (eq.~\ref{eq:gen_cp_decomp}) through:
\begin{itemize}
\item $a_1^y=1$
\item $\aaa^{1,i}=\left\{
	\begin{array}{ll}
		\tilde{F}^{-1}\left(\prod\nolimits_{l=0}^{L-1}r_{l}\cdot\1+\oo(\epsilon)\right)  & ,i=1 \\
		\0                                                                                                     & ,i>1
	\end{array}
\right.$
\end{itemize}
leads to $\A\left(h_y^S\right)=\A\left(h_y^D\right)$, as required.
\end{proof}

Comparing claims~\ref{claim:prod_depth_eff_complete} and~\ref{claim:max_depth_eff_incomplete}, we see that depth efficiency is complete under linear activation with product pooling, and incomplete under ReLU activation with max pooling.
We interpret this as indicating that \textbf{\emph{convolutional arithmetic circuits benefit from the expressive power of depth more than convolutional rectifier networks do}}.
This result is rather surprising, especially given the fact that convolutional rectifier networks are much more commonly used in practice.
We attribute the discrepancy primarily to historical reasons, and conjecture that developing effective methods for training convolutional arithmetic circuits, thereby fulfilling their expressive potential, may give rise to a deep learning architecture that is provably superior to convolutional rectifier networks but has so far been overlooked by practitioners.

Loosely speaking, we have shown that the gap in expressive power between the shallow and deep ConvNets is greater with linear activation and product pooling than it is with ReLU activation and max pooling.
One may wonder at this point if it is plausible to deduce from this which architectural setting is more expressive, as a-priori, altering the shallow \vs deep ConvNet comparisons such that one network has linear activation with product pooling and the other has ReLU activation with max pooling, may change the expressive gaps in favor of the latter.
Claims~\ref{claim:prod_depth_eff_complete_cross} and~\ref{claim:max_depth_eff_incomplete_cross} below show that this is not the case.
Specifically, they show that the depth efficiency of the deep ConvNet with linear activation and product pooling remains complete when the shallow ConvNet has ReLU activation and max pooling (claim~\ref{claim:prod_depth_eff_complete_cross}), and on the other hand, the depth efficiency of the deep ConvNet with ReLU activation and max pooling remains incomplete when the shallow ConvNet has linear activation and product pooling (claim~\ref{claim:max_depth_eff_incomplete_cross}).
This affirms our stand regarding the expressive advantage of convolutional arithmetic circuits over convolutional rectifier networks.

\begin{claim} \label{claim:prod_depth_eff_complete_cross}
Let $f_{\theta_1}{\ldots}f_{\theta_M}$ be any set of linearly independent representation functions for a deep ConvNet with linear activation and product pooling.
Suppose we randomize the weights of the network by some continuous distribution.
Then, with probability~1, we obtain score functions that cannot be realized by a shallow ConvNet with ReLU activation and max pooling if the number of hidden channels in the latter ($Z$) is less than $\min\{r_0,M\}^{\nicefrac{N}{2}}\cdot\frac{2}{M{\cdot}N}$.
\end{claim}

\begin{proof}
The proof here follows readily from those of claims~\ref{claim:prod_depth_eff_complete} and~\ref{claim:max_depth_eff_exist}.
Namely, in the proof of claim~\ref{claim:prod_depth_eff_complete} we state that for templates $\x^{(1)}\ldots\x^{(M)}\in\R^s$ chosen such that $F$ is invertible (these exist according to claim~\ref{claim:F_inv_x4f}), a grid tensor produced by the deep ConvNet with linear activation and product pooling, when arranged as a matrix, has rank at least $\min\{r_0,M\}^{\nicefrac{N}{2}}$ for all linear weight ($\aaa^{l,j,\gamma}$) settings but a set of measure zero.
That is to say, a matrix produced by the matricized generalized HT decomposition (eq.~\ref{eq:mat_gen_ht_decomp}) with $g(a,b)=a{\cdot}b$, has rank at least $\min\{r_0,M\}^{\nicefrac{N}{2}}$ for all weight ($\aaa^{l,j,\gamma}$) settings but a set of measure zero.
On the other hand, we have shown in the proof of claim~\ref{claim:max_depth_eff_exist} that a shallow ConvNet with ReLU activation and max pooling generates grid tensors that when arranged as matrices, have rank at most $Z{\cdot}\frac{M{\cdot}N}{2}$.
More specifically, we have shown that the matricized generalized CP decomposition (eq.~\ref{eq:mat_gen_cp_decomp}) with $g(a,b)=\max\{a,b,0\}$ produces matrices with rank at most $Z{\cdot}\frac{M{\cdot}N}{2}$.
This implies that under almost all linear weight ($\aaa^{l,j,\gamma}$) settings for a deep ConvNet with linear activation and product pooling, the generated grid tensor cannot be replicated by a shallow ConvNet with ReLU activation and max pooling if the latter has less than $Z=\min\{r_0,M\}^{\nicefrac{N}{2}}\cdot\frac{2}{M{\cdot}N}$ hidden channels.
\end{proof}

\begin{claim} \label{claim:max_depth_eff_incomplete_cross}
Suppose we randomize the weights of a deep ConvNet with ReLU activation and max pooling by some continuous distribution with non-vanishing continuous probability density function.
Then, assuming covering templates exist, with positive probability, we obtain score functions that can be realized by a shallow ConvNet with linear activation and product pooling having only a single hidden channel ($Z=1$).
\end{claim}

\begin{proof}
The proof here is almost identical to that of claim~\ref{claim:max_depth_eff_incomplete}.
The only difference is where we show that a tensor as in eq.~\ref{eq:max_deep_trivial_tensor} may be realized by the generalized CP decomposition (eq.~\ref{eq:gen_cp_decomp}) with $Z=1$.
In the proof of claim~\ref{claim:max_depth_eff_incomplete} the underlying operation of the decomposition was $g(a,b)=\max\{a,b,0\}$ (corresponding to ReLU activation and max pooling), whereas here it is $g(a,b)=a{\cdot}b$ (corresponding to linear activation and product pooling).
To account for this difference, we again assume that $\tilde{F}$~--~the matrix $F$ of the generalized CP decomposition, is invertible (non-degeneracy ensures that this may be achieved with proper choice of representation functions for the shallow ConvNet), and modify the decomposition's weight setting as follows:
\begin{itemize}
\item $a_1^y=1$
\item $\aaa^{1,i}=\left\{
	\begin{array}{ll}
		\tilde{F}^{-1}\left(\prod\nolimits_{l=0}^{L-1}r_{l}\cdot\1+\oo(\epsilon)\right)  & ,i=1 \\
		\tilde{F}^{-1}\1                                                                                   & ,i>1
	\end{array}
\right.$
\end{itemize}
This leads to $\A\left(h_y^S\right)=\A\left(h_y^D\right)$, as required.
\end{proof}

\subsubsection{Approximation} \label{sec:analysis:depth_eff:approx}

In their current form, the results in our analysis establishing depth efficiency (claims~\ref{claim:prod_depth_eff_complete},~\ref{claim:max_depth_eff_exist},~\ref{claim:prod_depth_eff_complete_cross} and the analogous ones in sec.~\ref{sec:analysis:shared_coeff}) relate to exact realization.
Specifically, they provide a lower bound on the size of a shallow ConvNet required in order for it to \emph{realize exactly} a grid tensor generated by a deep ConvNet.
From a practical perspective, a more interesting question would be the size required by a shallow ConvNet in order to \emph{approximate} the computation of a deep ConvNet.
A-priori, it may be that although the size required for exact realization is exponential, the one required for approximation is only polynomial.
As we briefly discuss below, this is not the case, and in fact all of the lower bounds we have provided apply not only to exact realization, but also to arbitrarily-well approximation.

When proving that a grid tensor generated by a shallow ConvNet beneath a certain size cannot be equal to a grid tensor generated by a deep ConvNet, we always rely on matricization rank.
Namely, we arrange the grid tensors as matrices, and derive constants $R,r\in\N$, $R>r$, such that the matrix corresponding to the deep ConvNet has rank at least $R$, while that corresponding to the shallow ConvNet has rank at most $r$.
While used in our proofs solely to show that the matrices are different, this actually entails information regarding the distance between them. 
Namely, if we denote the singular values of the matrix corresponding to the deep ConvNet by $\sigma_1\geq\sigma_2\geq\ldots\geq0$, the squared Euclidean (Frobenius) distance between the matrices is at least $\sigma_{r+1}^2+\cdots+\sigma_{R}^2$.
Since the matrices are merely rearrangements of the grid tensors, we have a lower bound on the distance between the shallow ConvNet's grid tensor and the target grid tensor generated by the deep ConvNet, so in particular arbitrarily-well approximation is not possible.

\subsubsection{On the Incidence of Depth Efficiency} \label{sec:analysis:depth_eff:depth_eff_incidence}

In claim~\ref{claim:prod_depth_eff_complete} we saw that depth efficiency is complete with linear activation and product pooling.
That is to say, with linear activation and product pooling, besides a negligible set, all weight settings for the deep ConvNet (fig.~\ref{fig:convnet} with size-$2$ pooling windows and $L=\log_{2}N$ hidden layers) lead to score functions that cannot be realized by the shallow ConvNet (fig.~\ref{fig:shallow_convnet}) unless the latter has super-polynomial size.
We have also seen (claims~\ref{claim:max_depth_eff_exist} and~\ref{claim:max_depth_eff_incomplete}) that replacing the activation and pooling operators by ReLU and max respectively, makes depth efficiency incomplete.
There are still weight settings leading the deep ConvNet to generate score functions that require the shallow ConvNet to have super-polynomial size, but these do not occupy the entire space.
In other words, there is now positive measure to the set of deep ConvNet weight configurations leading to score functions efficiently realizable by the shallow ConvNet.
A natural question would then be just how frequent depth efficiency is under ReLU activation and max pooling.
More formally, we may consider a uniform distribution over a compact domain in the deep ConvNet's weight space, and ask the following.
Assuming weights for the deep ConvNet are drawn from this distribution, what is the probability that generated score functions exhibit depth efficiency, \ie require super-polynomial size from the shallow ConvNet?
In the following we address this question, arguing that the probability tends to $1$ as the number of channels in the hidden layers of the deep ConvNet grows.
We do not prove this formally, but nonetheless provide a framework we believe may serve as a basis for establishing formal results concerning the incidence of depth efficiency.
The framework is not limited to ReLU activation and max pooling~--~it may be used under different choices of activation and pooling operators as well.

\medskip

The central tool used in the paper for proving depth efficiency is the rank of grid tensors when these are arranged as matrices.
We establish upper bounds on the rank of matricized grid tensors produced by the shallow ConvNet through the matricized generalized CP decomposition (eq.~\ref{eq:mat_gen_cp_decomp}).
These upper bounds are typically linear in the size of the input ($N$) and the number of hidden channels in the network ($Z$).
The challenge is then to derive a super-polynomial (in $N$) lower bound on the rank of matricized grid tensors produced by the deep ConvNet through the matricized generalized HT decomposition (eq.~\ref{eq:mat_gen_ht_decomp}).
In the case of linear activation and product pooling ($g(a,b)=a{\cdot}b$), the generalized Kronecker product $\odotg$ reduces to the standard Kronecker product $\odot$, and the rank-multiplicative property of the latter ($rank(A{\odot}B)=rank(A){\cdot}rank(B)$) can be used to show (see~\cite{\expresstensors}) that besides in negligible (zero measure) cases, rank grows rapidly through the levels of the matricized generalized HT decomposition (eq.~\ref{eq:mat_gen_ht_decomp}), to the point where the final produced matrix has exponential rank.
This situation does not persist when the activation and pooling operators are replaced by ReLU and max (respectively).
Indeed, in the proof of claim~\ref{claim:max_depth_eff_incomplete} we explicitly presented a non-negligible (positive measure) case where the matricized generalized HT decomposition (eq.~\ref{eq:mat_gen_ht_decomp}) produces a matrix of rank~$1$.
To study the incidence of depth efficiency under ReLU activation and max pooling, we assume the weights ($\aaa^{l,j,\gamma}$) of the matricized generalized HT decomposition (eq.~\ref{eq:mat_gen_ht_decomp}) are drawn independently and uniformly from a bounded interval (\eg~$[-1,1]$), and question the probability of the produced matrix $[\A\left(h_y^D\right)]$ having rank super-polynomial in~$N$.

To study $rank[\A\left(h_y^D\right)]$, we sequentially traverse through the levels $l=1{\ldots}L$ of the matricized generalized HT decomposition (eq.~\ref{eq:mat_gen_ht_decomp}), at each level going over all locations $j\in[N/2^l]$.
When at location $j$ of level $l$, for each $\alpha\in[r_{l-1}]$, we draw the weights $\aaa^{l-1,2j-1,\alpha}$ and $\aaa^{l-1,2j,\alpha}$ (independently of the previously drawn weights), and observe the random variable $R^{l,j,\alpha}$, defined as the rank of the matrix $[\phi^{l-1,2j-1,\alpha}]\odotg[\phi^{l-1,2j,\alpha}]$.
Given the weights drawn while traversing through the previous levels of the decomposition, the random variables $\{R^{l,j,\alpha}\in\N\}_{\alpha\in[r_{l-1}]}$ are independent and identically distributed.
The random variable $R^{l,j}:=\max_{\alpha\in[r_{l-1}]}\{R^{l,j,\alpha}\}$ thus tends to concentrate on higher and higher values as $r_{l-1}$ (number of channels in hidden layer $l-1$ of the deep ConvNet) grows.
When the next level ($l+1$) of the decomposition will be traversed, the weights $\{\aaa^{l,j,\gamma}\}_{\gamma\in[r_l]}$ will be drawn, and the matrices $\{[\phi^{l,j,\gamma}]\}_{\gamma\in[r_l]}$ will be generated.
According to claim~\ref{claim:mat_sum_rank} below, with probability $1$, all of these matrices will have rank equal to at least $R^{l,j}$.
We conclude that, assuming the generalized Kronecker product $\odotg$ has the potential of increasing the rank of its operands, ranks will generally ascend across the levels of the matricized generalized HT decomposition (eq.~\ref{eq:mat_gen_ht_decomp}), with steeper ascends being more and more probable as the number of channels in the hidden layers of the deep ConvNet ($r_0{\ldots}r_{L-1}$) grows.

\begin{figure}
\includegraphics[width=\columnwidth]{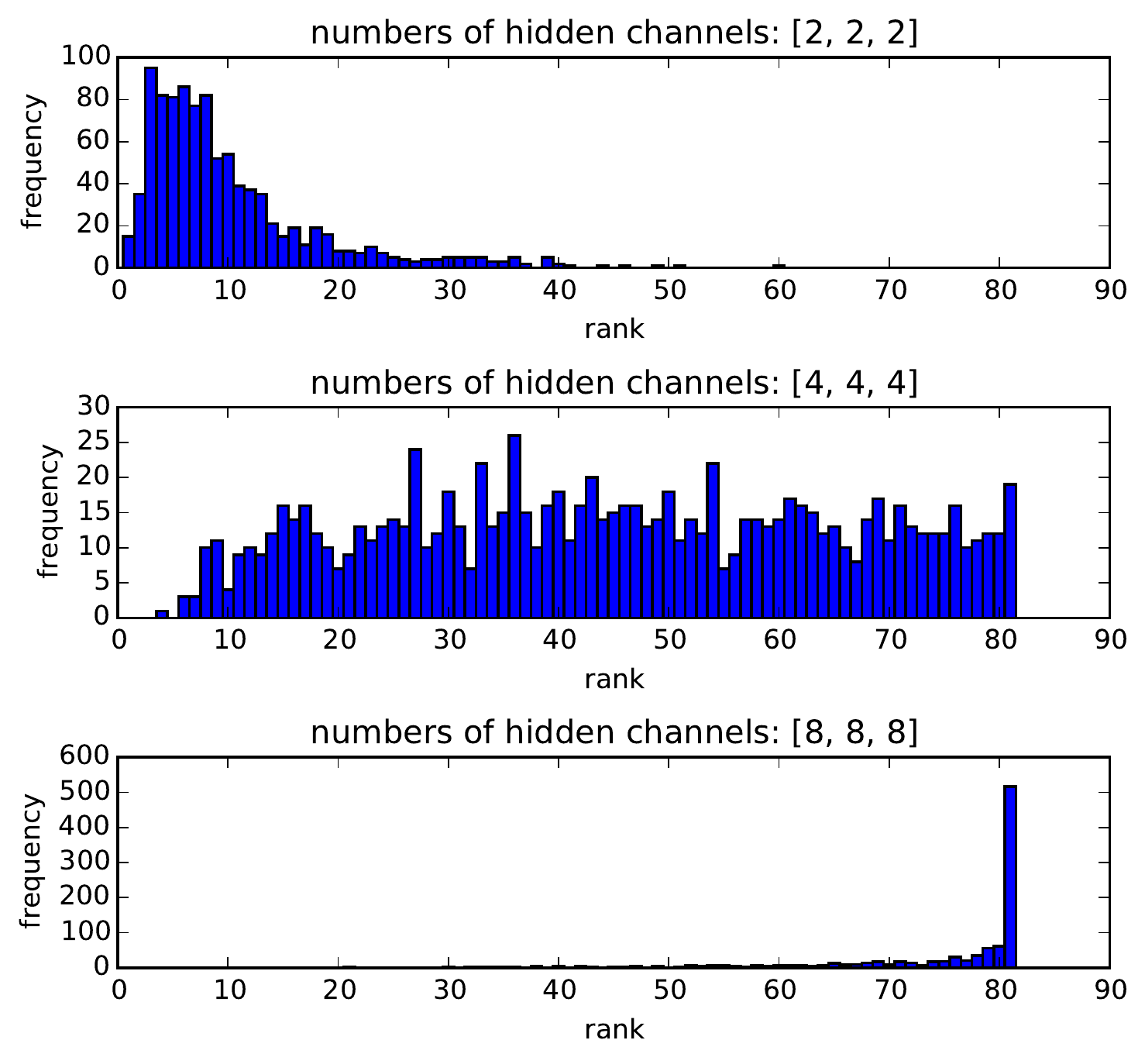}
\caption{
Simulation results demonstrating that under ReLU activation and max pooling, the incidence of depth efficiency increases as the number of channels in the hidden layers of the deep ConvNet ($r_0{\ldots}r_{L-1}$) grows.
The plots show histograms of the ranks produced by the matricized generalized HT decomposition (eq.~\ref{eq:mat_gen_ht_decomp}) with $g(a,b)=\max\{a,b,0\}$.
The number of levels in the decomposition was set to $L=3$ (implying input size of $N=2^L=8$).
The size of the representation matrix $F$ was set through $M=3$, and the matrix itself was fixed to the identity.
Weights ($\aaa^{l,j,\gamma}$) were drawn at random independently and uniformly from the interval $[-1,1]$.
Three channel-width configurations were tried: (i)~$r_0=r_1=r_2=2$ (ii)~$r_0=r_1=r_2=4$ (ii)~$r_0=r_1=r_2=8$.
For each configuration $1000$ random tests were run, creating the histograms presented in the figure (each test produced a single matrix $[\A(h_y^D)]$, accounting for a single entry in a histogram).
As can be seen, the distribution of the produced rank ($rank[\A(h_y^D)]$) tends towards the maximum ($M^{N/2}=81$) as the numbers of hidden channels grow.
}
\label{fig:relu_max_ht_rank_sim}
\end{figure}

The main piece that is missing in order to complete the sketch we have outlined above into a formal proof, is the behavior of rank under the generalized Kronecker product $\odotg$.
This obviously depends on the choice of underlying operator $g$.
In the case of linear activation and product pooling $g(a,b)=a{\cdot}b$, the generalized Kronecker product $\odotg$ reduces to the standard  Kronecker product $\odot$, and ranks always increase multiplicatively, \ie~$rank(A{\odot}B)=rank(A){\cdot}rank(B)$ for any matrices $A$ and $B$.
The fact that there is a simple law governing the behavior of ranks makes this case relatively simple to analyze, and we indeed have a full characterization (claim~\ref{claim:prod_depth_eff_complete}).
In the case of \emph{linear} activation and max pooling the underlying operator is given by $g(a,b)=\max\{a,b\}$, and it is not difficult to see that $\odotg$ does not decrease rank, \ie $rank(A{\odotg}B){\geq}\min\{rank(A),rank(B)\}$ for any matrices $A$ and $B$
\footnote{
To see this, simply note that under the choice $g(a,b)=\max\{a,b\}$ there is either a sub-matrix of $A{\odotg}B$ that is equal to $A$, or one that is equal to $B$.
}.
For ReLU activation and max pooling, corresponding to the choice $g(a,b)=\max\{a,b,0\}$, there is no simple rule depicting the behavior of ranks under $\odotg$, and in fact, for matrices $A$ and $B$ holding negative values, the rank of $rank(A{\odotg}B)$ necessarily drops to zero.
Nonetheless, it seems reasonable to assume that at least in some cases, a non-linear operation such as $\odotg$ does increase rank, and as we have seen, benefiting from these cases is more probable when the hidden layers of the deep ConvNet include many channels.
To this end, we provide in fig.~\ref{fig:relu_max_ht_rank_sim} simulation results for the case of ReLU activation and max pooling ($g(a,b)=\max\{a,b,0\}$), demonstrating that indeed ranks produced by the matricized generalized HT decomposition (eq.~\ref{eq:mat_gen_ht_decomp}) tend to be higher as $r_0{\ldots}r_{L-1}$ grow.
We leave a complete formal analysis of this phenomenon to future work.

\begin{claim} \label{claim:mat_sum_rank}
Let $A_1{\ldots}A_m$ be given matrices of the same size, having ranks $r_1{\ldots}r_m$ respectively.
For every weight vector $\alphabf\in\R^m$ define the matrix $A(\alphabf):=\sum_{i=1}^{m}\alpha_{i}A_i$, and suppose we randomize $\alphabf$ by some continuous distribution.
Then, with probability $1$, we obtain a matrix $A(\alphabf)$ having rank at least $\max_{i\in[m]}r_i$.
\end{claim}

\begin{proof}
Our proof relies on concepts and results from Lebesgue measure theory (see sec.~\ref{sec:analysis:prelim} for a brief discussion).
The result to prove is equivalent to stating that there is measure zero to the set of weight vectors $\alphabf$ for which $rank(A(\alphabf))<\max_{i\in[m]}r_i$.

Assume without loss of generality that $\max_{i\in[m]}r_i$ is equal to $r_1$, and that the top-left $r_1{\times}r_1$ block of $A_1$ is non-singular.
For every $\alphabf$ define $p(\alphabf):=\det(A(\alphabf)_{1:r_1,1:r_1})$, \ie $p(\alphabf)$ is the determinant of the $r_1{\times}r_1$ top-left block of the matrix $A(\alphabf)$.
$p(\alphabf)$ is obviously a polynomial in the entries of $\alphabf$, and by assumption $p(\e_1)\neq0$, where $\e_1\in\R^m$ is the vector holding $1$ in its first entry and $0$ elsewhere.
Since a non-zero polynomial vanishes only on a set of zero measure (see~\cite{caron2005zero} for example), the set of weight vectors $\alphabf$ for which $p(\alphabf)=0$ has measure zero.
This implies that the top-left $r_1{\times}r_1$ block of $A(\alphabf)$ is non-singular almost everywhere, and in particular $rank(A(\alphabf)){\geq}r_1=\max_{i\in[m]}r_i$ almost everywhere.
\end{proof}

\subsection{Shared Coefficients for Convolution} \label{sec:analysis:shared_coeff}

To this end, our analysis has focused on the unshared setting, where the coefficients of the $1\times1$ conv filters in our networks (fig.~\ref{fig:convnet}) may vary across spatial locations.
In practice, ConvNets typically enforce sharing, which in our framework implies that the coefficients of the $1\times1$ conv filter in channel~$\gamma$ of hidden layer~$l$, are the same for all locations~$j$.
In this subsection we analyze the shared setting, following a line similar to that of our analysis for the unshared setting given in sec.~\ref{sec:analysis:universality} and~\ref{sec:analysis:depth_eff}.
For brevity, we assume the reader is familiar with the latter, and do not repeat discussions given there.

\medskip

In the shared setting, the shallow ConvNet (fig.~\ref{fig:shallow_convnet}) would have a single weight vector $\aaa^z$ for every hidden channel $z$, as opposed to the unshared setting where it had a weight vector $\aaa^{z,i}$ for every location $i$ in every hidden channel $z$.
Grid tensors produced by the shallow ConvNet in the shared setting are given by what we call the \emph{shared generalized CP decomposition}:
\be
\A\left(h_y^S\right) = \sum_{z=1}^Z a_z^y \cdot \underbrace{(F\aaa^{z}) \otimesg \cdots \otimesg (F\aaa^{z})}_{N~\text{times}}
\label{eq:shared_gen_cp_decomp}
\ee
As for the deep ConvNet (fig.~\ref{fig:convnet} with size-$2$ pooling windows and $L=\log_{2}N$ hidden layers), in the shared setting, instead of having a weight vector $\aaa^{l,j,\gamma}$ for every hidden layer $l$, channel $\gamma$ and location $j$, there is a single weight vector $\aaa^{l,\gamma}$ for all locations of channel $\gamma$ in hidden layer $l$.
Produced grid tensors are then given by the \emph{shared generalized HT decomposition}:
\bea
\phi^{1,\gamma} &=& \sum_{\alpha=1}^{r_0} a_\alpha^{1,\gamma} 
(F\aaa^{0,\alpha}) \otimesg  (F\aaa^{0,\alpha})
\nonumber \\
&\cdots& 
\nonumber\\
\phi^{l,\gamma} &=& \sum_{\alpha=1}^{r_{l-1}} a_\alpha^{l,\gamma} 
\underbrace{\phi^{l-1,\alpha}}_{\text{order $2^{l-1}$}} \otimesg  
\underbrace{\phi^{l-1,\alpha}}_{\text{order $2^{l-1}$}} 
\nonumber\\
&\cdots& 
\nonumber\\
\phi^{L-1,\gamma} &=& \sum_{\alpha=1}^{r_{L-2}} a_\alpha^{L-1,\gamma} 
\underbrace{\phi^{L-2,\alpha}}_{\text{order $\frac{N}{4}$}} \otimesg  
\underbrace{\phi^{L-2,\alpha}}_{\text{order $\frac{N}{4}$}}  
\nonumber\\ 
\A\left(h_y^D\right) &=& \sum_{\alpha=1}^{r_{L-1}} a_\alpha^{L,y} 
\underbrace{\phi^{L-1,\alpha}}_{\text{order $\frac{N}{2}$}} \otimesg  
\underbrace{\phi^{L-1,\alpha}}_{\text{order $\frac{N}{2}$}}  
\label{eq:shared_gen_ht_decomp} 
\eea
We now turn to analyze universality and depth efficiency in the shared setting.

\subsubsection{Universality} \label{sec:analysis:shared_coeff:universality}

In the unshared setting we saw (sec.~\ref{sec:analysis:universality}) that linear activation with product pooling and ReLU activation with max pooling both lead to universality, whereas ReLU activation with average pooling does not.
We will now see that in the shared setting, no matter how the activation and pooling operators are chosen, universality is never met.

A shallow ConvNet with shared weights produces grid tensors through the shared generalized CP decomposition (eq.~\ref{eq:shared_gen_cp_decomp}).
A tensor $\A$ generated by this decomposition is necessarily \emph{symmetric}, \ie for any permutation $\delta:[N]\to[N]$ and indexes $d_1{\ldots}d_N$ it meets: $\A_{d_1{\ldots}d_N}=\A_{\delta(d_1){\ldots}\delta(d_N)}$.
Obviously not all tensors share this property, so indeed a shallow ConvNet with weight sharing is not universal.
A deep ConvNet with shared weights produces grid tensors through the shared generalized HT decomposition (eq.~\ref{eq:shared_gen_ht_decomp}).
For this decomposition, a generated tensor $\A$ is invariant to replacing the first and second halves of its modes, \ie for any indexes $d_1{\ldots}d_N$ it meets: $\A_{d_1,{\ldots},d_N}=\A_{d_{\nicefrac{N}{2}+1},\ldots,d_N,d_1,\ldots,d_{\nicefrac{N}{2}}}$.
Although this property is much less stringent than symmetry, it is still not met by most tensors, and so a deep ConvNet with weight sharing is not universal either.

\subsubsection{Depth Efficiency} \label{sec:analysis:shared_coeff:depth_eff}

Depth efficiency deals with the computational complexity of replicating a deep network's function using a shallow network.
In order for this question to be applicable, we require that the shallow network be a universal machine.
If this is not the case, then it is generally likely that the deep network's function simply lies outside the reach of the shallow network, and we do not obtain a quantitative insight into the true power of depth.
Since our shallow ConvNets are not universal with shared weights (sec.~\ref{sec:analysis:shared_coeff:universality}), we evaluate depth efficiency of deep ConvNets with shared weights against shallow ConvNets with \emph{unshared} weights.
Specifically, we do this for the activation-pooling choices leading shallow ConvNets with unshared weights to be universal: linear activation with product pooling, and ReLU activation with max pooling (see sec.~\ref{sec:analysis:universality}).

For linear activation with product pooling, the following claim, which is essentially a derivative of theorem~1 in~\cite{\expresstensors}, tells us that in the shared setting, as in the unshared setting, depth efficiency holds completely:

\begin{claim}[shared analogy of claim~\ref{claim:prod_depth_eff_complete}] \label{claim:prod_depth_eff_complete_shared}
Let $f_{\theta_1}{\ldots}f_{\theta_M}$ be any set of linearly independent representation functions for a deep ConvNet with linear activation, product pooling and weight sharing.
Suppose we randomize the weights of the network by some continuous distribution.
Then, with probability~1, we obtain score functions that cannot be realized by a shallow ConvNet with linear activation and product pooling (\emph{not} limited by weight sharing), if the number of hidden channels in the latter ($Z$) is less than $\min\{r_0,M\}^{\nicefrac{N}{2}}$.
\end{claim}

\begin{proof}
The proof here is almost identical to that of claim~\ref{claim:prod_depth_eff_complete}.
The only difference is that in the latter, we used the fact that the generalized HT decomposition (eq.~\ref{eq:gen_ht_decomp}), when equipped with $g(a,b)=a{\cdot}b$, almost always produces tensors whose matrix arrangements have rank at least $\min\{r_0,M\}^{\nicefrac{N}{2}}$, whereas here, we require an analogous result for the \emph{shared} generalized HT decomposition (eq.~\ref{eq:shared_gen_ht_decomp}).
Such result is provided by the proof of theorem~1 in~\cite{\expresstensors}.
\end{proof}

Heading on to ReLU activation and max pooling, we will show that here too, the situation in the shared setting is the same as in the unshared setting.
Specifically, depth efficiency holds, but not completely.
We prove this via two claims, analogous to claims~\ref{claim:max_depth_eff_exist} and~\ref{claim:max_depth_eff_incomplete} in sec.~\ref{sec:analysis:depth_eff}:

\begin{claim}[shared analogy of claim~\ref{claim:max_depth_eff_exist}] \label{claim:max_depth_eff_exist_shared}
There exist weight settings for a deep ConvNet with ReLU activation, max pooling and weight sharing, giving rise to score functions that cannot be realized by a shallow ConvNet with ReLU activation and max pooling (\emph{not} limited by weight sharing), if the number of hidden channels in the latter ($Z$) is less than $\min\{r_0,M\}^{\nicefrac{N}{2}}\cdot\frac{2}{M{\cdot}N}$.
\end{claim}

\begin{proof}
In the proof of claim~\ref{claim:max_depth_eff_exist} we have shown, for arbitrary distinct templates $\x^{(1)}\ldots\x^{(M)}\in\R^s$, an explicit weight setting for the deep ConvNet with ReLU activation and max pooling, leading the latter to produce a grid tensor that cannot be realized by a shallow ConvNet with ReLU activation and max pooling, if that has less than $\min\{r_0,M\}^{\nicefrac{N}{2}}\cdot\frac{2}{M{\cdot}N}$ hidden channels.
Since the given weight setting was location invariant, \ie the assignment of $\aaa^{l,j,\gamma}$ did not depend on $j$, it applies as is to a deep ConvNet with weight sharing, and the desired result readily follows.
\end{proof}

\begin{claim}[shared analogy of claim~\ref{claim:max_depth_eff_incomplete}] \label{claim:max_depth_eff_incomplete_shared}
Suppose we randomize the weights of a deep ConvNet with ReLU activation, max pooling and weight sharing by some continuous distribution with non-vanishing continuous probability density function.
Then, assuming covering templates exist, with positive probability, we obtain score functions that can be realized by a shallow ConvNet with ReLU activation and max pooling having only a single hidden channel ($Z=1$).
\end{claim}

\begin{proof}
The proof is similar in spirit to that of claim~\ref{claim:max_depth_eff_incomplete}, which dealt with incompleteness of depth efficiency under ReLU activation and max pooling in the unshared setting.
Our focus here is on the shared setting, or more specifically, on the case where the deep ConvNet is limited by weight sharing while the shallow ConvNet is not.
Accordingly, we would like to show the following.
Fixing $g(a,b)=\max\{a,b,0\}$, per arbitrary invertible $F$ there exists a weight ($\aaa^{l,\gamma}$) setting for the shared generalized HT decomposition (eq.~\ref{eq:shared_gen_ht_decomp}), such that the produced tensor may be realized by the generalized CP decomposition (eq.~\ref{eq:gen_cp_decomp}) with $Z=1$, and this holds even if the weights $\aaa^{l,\gamma}$ and matrix $F$ are subject to small perturbations.

Before heading on to prove that a weight setting as above exists, we introduce a new definition that will greatly simplify our proof.
We refer to a tensor $\A$ of order $P$ and dimension $M$ in each mode as \emph{basic}, if there exists a vector $\uu\in\R^M$ with non-decreasing entries ($u_1{\leq}\ldots{\leq}u_M$), such that $\A=\uu\otimesg\cdots\otimesg\uu$ (\ie $\A$ is equal to the generalized tensor product of $\uu$ with itself $P$ times, with underlying operation $g(a,b)=\max\{a,b,0\}$).
A basic tensor can obviously be realized by the generalized CP decomposition (eq.~\ref{eq:gen_cp_decomp}) with $Z=1$ (given that non-degeneracy is used to ensure the latter's representation matrix is non-singular), and so it suffices to find a weight ($\aaa^{l,\gamma}$) setting for the shared generalized HT decomposition (eq.~\ref{eq:shared_gen_ht_decomp}) that gives rise to a basic tensor, and in addition, ensures that small perturbations on the weights $\aaa^{l,\gamma}$ and matrix $F$ still yield basic tensors.
Two trivial facts that relate to basic tensors and will be used in our proof are: (i) the generalized tensor product of a basic tensor with itself is basic, and (ii) a linear combination of basic tensors with non-negative weights is basic.

Turning to the main part of the proof, we now show that the following weight setting meets our requirement:
\begin{itemize}
\item $\aaa^{0,\gamma}=F^{-1}\vv$
\item $\aaa^{l,\gamma}=\1$,~~$l\in[L-1]$
\item $\aaa^{L,y}=\1$
\end{itemize}
$\vv$ here stands for the vector $[1,2,\ldots,M]^\top\in\R^M$, and $\1$ is an all-$1$ vector with dimension to be understood by context.
Let $\E^F$ be an additive noise matrix applied to $F$, and $\{\epsbf^{l,\gamma}\}_{l,\gamma}$ be additive noise vectors applied to $\{\aaa^{l,\gamma}\}_{l,\gamma}$.
We would like to prove that under the weight setting above, when applied noise ($\E^F,\epsbf^{l,\gamma}$) is small enough, the grid tensor produced by the shared generalized HT decomposition (eq.~\ref{eq:shared_gen_ht_decomp}) is basic.

For convenience, we adopt the notation $\oo(\epsilon)$ as referring to vectors that tend to $\0$ as $\E^F\to0$ and $\epsbf^{l,\gamma}\to\0$, with the dimension of a vector to be understood by context.
Plugging in the noisy variables into the shared generalized HT decomposition (eq.~\ref{eq:shared_gen_ht_decomp}), we get for every $\alpha\in[r_0]$:
\beas
&((F+\E^F)(\aaa^{0,\alpha}+\epsbf^{0,\alpha}))\otimesg((F+\E^F)(\aaa^{0,\alpha}+\epsbf^{0,\alpha}))& \\
&=((F+\E^F)(F^{-1}\vv+\epsbf^{0,\alpha}))\otimesg((F+\E^F)(F^{-1}\vv+\epsbf^{0,\alpha}))& \\
&=\tilde{\vv}^\alpha\otimesg\tilde{\vv}^\alpha
\eeas
where $\tilde{\vv}^\alpha=\vv+\oo(\epsilon)$.
If the applied noise ($\E^F,\epsbf^{l,\gamma}$) is small enough the entries of $\tilde{\vv}^\alpha$ are non-decreasing and $\tilde{\vv}^\alpha\otimesg\tilde{\vv}^\alpha$ is a basic tensor (matrix).
Moving to the next level of the decomposition, we have for every $\gamma\in[r_1]$:
$$\phi^{1,\gamma}=\sum_{\alpha=1}^{r_0}(a_\alpha^{1,\gamma}+\epsilon_\alpha^{1,\gamma})\cdot
\tilde{\vv}^\alpha\otimesg\tilde{\vv}^\alpha$$
When applied noise ($\E^F,\epsbf^{l,\gamma}$) is small enough the weights of this linear combination are non-negative, and together with the tensors (matrices) $\tilde{\vv}^\alpha\otimesg\tilde{\vv}^\alpha$ being basic, this leads $\phi^{1,\gamma}$ to be basic as well.
Continuing in this fashion over the levels of the decomposition, we get that with small enough noise, for every $l\in[L-1]$ and $\gamma\in[r_l]$, $\phi^{l,\gamma}$ is a basic tensor.
A final step in this direction shows that under small noise, the produced grid tensor $\A\left(h_y^D\right)$ is basic as well.
This is what we set out to prove.
\end{proof}

\medskip

To recapitulate this subsection, we have shown that introducing weight sharing into the $1\times1$ conv operators of our networks, thereby limiting the general locally-connected linear mappings to be standard convolutions, disrupts universality, but leaves depth efficiency intact~--~it remains to hold completely under linear activation with product pooling, and incompletely under ReLU activation with max pooling.

\section{Discussion} \label{sec:discussion}

The contribution of this paper is twofold.
First, we introduce a construction in the form of \emph{generalized tensor decompositions}, that enables transforming convolutional arithmetic circuits into \emph{convolutional rectifier networks} (ConvNets with ReLU activation and max or average pooling).
This opens the door to various mathematical tools from the world of arithmetic circuits, now available for analyzing convolutional rectifier networks.
As a second contribution, we make use of such tools to prove new results on the expressive properties that drive this important class of networks.

Our analysis shows that convolutional rectifier networks are universal with max pooling, but not with average pooling.
This implies that if non-linearity originates solely from ReLU activation, increasing network size alone is not sufficient for expressing arbitrary functions.
More interestingly, we analyze the behavior of convolutional rectifier networks in terms of \emph{depth efficiency}, \ie of cases where a function generated by a deep network of polynomial size requires shallow networks to have super-polynomial size.
It is known that convolutional arithmetic circuits exhibit \emph{complete depth efficiency}, meaning that besides a negligible (zero measure) set, all functions generated by deep networks of this type are depth efficient.
We show that this is not the case with convolutional rectifier networks, for which depth efficiency exists, but is weaker in the sense that it is not complete (there is positive measure to the set of functions generated by a deep network that may be efficiently realized by shallow networks).

Depth efficiency is believed to be the key factor behind the success of deep learning.
Our analysis indicates that from this perspective, the widely used convolutional rectifier networks are inferior to convolutional arithmetic circuits.
This leads us to believe that convolutional arithmetic circuits bear the potential to improve the performance of deep learning beyond what is witnessed today.
Of course, a practical machine learning model is measured not only by its expressive power, but also by our ability to train it.
Over the years, massive amounts of research have been devoted to training convolutional rectifier networks.
Convolutional arithmetic circuits on the other hand received far less attention, although they have been successfully trained in recent works on the SimNet architecture~(\cite{\simnets,\deepsimnets}), demonstrating how the enhanced expressive power can lead to state of the art performance in computationally limited settings.

We believe that developing effective methods for training convolutional arithmetic circuits, thereby fulfilling their expressive potential, may give rise to a deep learning architecture that is provably superior to convolutional rectifier networks but has so far been overlooked by practitioners.

\ifdefined\CAMREADY
	\subsection*{Acknowledgments}
	This work is partly funded by Intel grant ICRI-CI no. 9-2012-6133 and by ISF Center grant 1790/12.  
	Nadav Cohen is supported by a Google Fellowship in Machine Learning.
\fi

\subsection*{References}
\small{
\bibliographystyle{plainnat}
\bibliography{refs.bib}

\begin{thebibliography}{32}
\providecommand{\natexlab}[1]{#1}
\providecommand{\url}[1]{\texttt{#1}}
\expandafter\ifx\csname urlstyle\endcsname\relax
  \providecommand{\doi}[1]{doi: #1}\else
  \providecommand{\doi}{doi: \begingroup \urlstyle{rm}\Url}\fi

\bibitem[Caron and Traynor(2005)]{caron2005zero}
Richard Caron and Tim Traynor.
\newblock {The zero set of a polynomial}.
\newblock \emph{WSMR Report 05-02}, 2005.

\bibitem[Clark and Storkey(2014)]{clark2014teaching}
Christopher Clark and Amos Storkey.
\newblock Teaching deep convolutional neural networks to play go.
\newblock \emph{arXiv preprint arXiv:1412.3409}, 2014.

\bibitem[Cohen and Shashua(2014)]{cohen2014simnets}
Nadav Cohen and Amnon Shashua.
\newblock {SimNets: A Generalization of Convolutional Networks}.
\newblock \emph{NIPS Deep Learning and Representation Learning Workshop}, 2014.

\bibitem[Cohen et~al.(2015{\natexlab{a}})Cohen, Sharir, and
  Shashua]{cohen2015deep}
Nadav Cohen, Or~Sharir, and Amnon Shashua.
\newblock {Deep SimNets}.
\newblock \emph{arXiv.org}, June 2015{\natexlab{a}}.

\bibitem[Cohen et~al.(2015{\natexlab{b}})Cohen, Sharir, and
  Shashua]{cohen2015expressive}
Nadav Cohen, Or~Sharir, and Amnon Shashua.
\newblock On the expressive power of deep learning: a tensor analysis.
\newblock \emph{arXiv preprint arXiv:1509.05009}, 2015{\natexlab{b}}.

\bibitem[Cybenko(1989)]{Cybenko:1989fm}
G~Cybenko.
\newblock {Approximation by superpositions of a sigmoidal function}.
\newblock \emph{Mathematics of Control, Signals and Systems}, 2\penalty0
  (4):\penalty0 303--314, 1989.

\bibitem[Delalleau and Bengio(2011)]{bengio2011shallow}
Olivier Delalleau and Yoshua Bengio.
\newblock Shallow vs. deep sum-product networks.
\newblock In \emph{Advances in Neural Information Processing Systems}, pages
  666--674, 2011.

\bibitem[Eldan and Shamir(2015)]{eldan2015power}
Ronen Eldan and Ohad Shamir.
\newblock The power of depth for feedforward neural networks.
\newblock \emph{arXiv preprint arXiv:1512.03965}, 2015.

\bibitem[Goodfellow et~al.(2016)Goodfellow, Bengio, and
  Courville]{Goodfellow-et-al-2016-Book}
Ian Goodfellow, Yoshua Bengio, and Aaron Courville.
\newblock Deep learning.
\newblock Book in preparation for MIT Press, 2016.
\newblock URL \url{http://goodfeli.github.io/dlbook/}.

\bibitem[Hackbusch and K{\"u}hn(2009)]{Hackbusch:2009jj}
W~Hackbusch and S~K{\"u}hn.
\newblock {A New Scheme for the Tensor Representation}.
\newblock \emph{Journal of Fourier Analysis and Applications}, 15\penalty0
  (5):\penalty0 706--722, 2009.

\bibitem[Hackbusch(2012)]{Hackbusch-book}
Wolfgang Hackbusch.
\newblock \emph{{Tensor Spaces and Numerical Tensor Calculus}}, volume~42 of
  \emph{Springer Series in Computational Mathematics}.
\newblock Springer Science {\&} Business Media, Berlin, Heidelberg, February
  2012.

\bibitem[Hornik et~al.(1989)Hornik, Stinchcombe, and White]{Hornik:1989fr}
Kurt Hornik, Maxwell~B Stinchcombe, and Halbert White.
\newblock {Multilayer feedforward networks are universal approximators.}
\newblock \emph{Neural networks}, 2\penalty0 (5):\penalty0 359--366, 1989.

\bibitem[Jones(2001)]{jones2001lebesgue}
Frank Jones.
\newblock \emph{Lebesgue integration on Euclidean space}.
\newblock Jones \& Bartlett Learning, 2001.

\bibitem[Kolda and Bader(2009)]{Kolda-Bader2009}
Tamara~G Kolda and Brett~W Bader.
\newblock {Tensor Decompositions and Applications.}
\newblock \emph{SIAM Review ()}, 51\penalty0 (3):\penalty0 455--500, 2009.

\bibitem[Krizhevsky et~al.(2012)Krizhevsky, Sutskever, and
  Hinton]{Krizhevsky:2012wl}
Alex Krizhevsky, Ilya Sutskever, and Geoffrey~E Hinton.
\newblock {ImageNet Classification with Deep Convolutional Neural Networks.}
\newblock \emph{Advances in Neural Information Processing Systems}, pages
  1106--1114, 2012.

\bibitem[LeCun and Bengio(1995)]{lecun1995convolutional}
Yann LeCun and Yoshua Bengio.
\newblock Convolutional networks for images, speech, and time series.
\newblock \emph{The handbook of brain theory and neural networks},
  3361\penalty0 (10), 1995.

\bibitem[LeCun et~al.(2015)LeCun, Bengio, and Hinton]{LeCun:2015dt}
Yann LeCun, Yoshua Bengio, and Geoffrey Hinton.
\newblock {Deep learning}.
\newblock \emph{Nature}, 521\penalty0 (7553):\penalty0 436--444, May 2015.

\bibitem[Leshno et~al.(1993)Leshno, Lin, Pinkus, and
  Schocken]{leshno1993multilayer}
Moshe Leshno, Vladimir~Ya Lin, Allan Pinkus, and Shimon Schocken.
\newblock Multilayer feedforward networks with a nonpolynomial activation
  function can approximate any function.
\newblock \emph{Neural networks}, 6\penalty0 (6):\penalty0 861--867, 1993.

\bibitem[Martens and Medabalimi(2014)]{martens2014expressive}
James Martens and Venkatesh Medabalimi.
\newblock On the expressive efficiency of sum product networks.
\newblock \emph{arXiv preprint arXiv:1411.7717}, 2014.

\bibitem[Montufar et~al.(2014)Montufar, Pascanu, Cho, and
  Bengio]{montufar2014number}
Guido~F Montufar, Razvan Pascanu, Kyunghyun Cho, and Yoshua Bengio.
\newblock On the number of linear regions of deep neural networks.
\newblock In \emph{Advances in Neural Information Processing Systems}, pages
  2924--2932, 2014.

\bibitem[Nair and Hinton(2010)]{nair2010rectified}
Vinod Nair and Geoffrey~E Hinton.
\newblock Rectified linear units improve restricted boltzmann machines.
\newblock In \emph{Proceedings of the 27th International Conference on Machine
  Learning (ICML-10)}, pages 807--814, 2010.

\bibitem[Pascanu et~al.(2013)Pascanu, Montufar, and Bengio]{pascanu2013number}
Razvan Pascanu, Guido Montufar, and Yoshua Bengio.
\newblock On the number of inference regions of deep feed forward networks with
  piece-wise linear activations.
\newblock \emph{arXiv preprint arXiv}, 1312, 2013.

\bibitem[Poggio et~al.(2015)Poggio, Anselmi, and Rosasco]{poggio2015theory}
Tomaso Poggio, Fabio Anselmi, and Lorenzo Rosasco.
\newblock I-theory on depth vs width: hierarchical function composition.
\newblock Technical report, Center for Brains, Minds and Machines (CBMM), 2015.

\bibitem[Poon and Domingos(2011)]{Poon-Domingos2011}
Hoifung Poon and Pedro Domingos.
\newblock Sum-product networks: A new deep architecture.
\newblock In \emph{Computer Vision Workshops (ICCV Workshops), 2011 IEEE
  International Conference on}, pages 689--690. IEEE, 2011.

\bibitem[Shen et~al.(2014)Shen, He, Gao, Deng, and Mesnil]{shen2014learning}
Yelong Shen, Xiaodong He, Jianfeng Gao, Li~Deng, and Gr{\'e}goire Mesnil.
\newblock Learning semantic representations using convolutional neural networks
  for web search.
\newblock In \emph{Proceedings of the companion publication of the 23rd
  international conference on World wide web companion}, pages 373--374.
  International World Wide Web Conferences Steering Committee, 2014.

\bibitem[Shpilka and Yehudayoff(2010)]{shpilka2010arithmetic}
Amir Shpilka and Amir Yehudayoff.
\newblock Arithmetic circuits: A survey of recent results and open questions.
\newblock \emph{Foundations and Trends in Theoretical Computer Science},
  5\penalty0 (3--4):\penalty0 207--388, 2010.

\bibitem[Simonyan and Zisserman(2014)]{simonyan2014very}
Karen Simonyan and Andrew Zisserman.
\newblock Very deep convolutional networks for large-scale image recognition.
\newblock \emph{arXiv preprint arXiv:1409.1556}, 2014.

\bibitem[Szegedy et~al.(2015)Szegedy, Liu, Jia, Sermanet, Reed, Anguelov,
  Erhan, Vanhoucke, and Rabinovich]{Szegedy:2014tb}
Christian Szegedy, Wei Liu, Yangqing Jia, Pierre Sermanet, Scott Reed, Dragomir
  Anguelov, Dumitru Erhan, Vincent Vanhoucke, and Andrew Rabinovich.
\newblock {Going Deeper with Convolutions}.
\newblock \emph{CVPR}, 2015.

\bibitem[Taigman et~al.(2014)Taigman, Yang, Ranzato, and Wolf]{Taigman:2014vs}
Yaniv Taigman, Ming Yang, Marc'Aurelio Ranzato, and Lior Wolf.
\newblock {DeepFace: Closing the Gap to Human-Level Performance in Face
  Verification}.
\newblock In \emph{CVPR '14: Proceedings of the 2014 IEEE Conference on
  Computer Vision and Pattern Recognition}. ~IEEE Computer Society, June 2014.

\bibitem[Telgarsky(2016)]{telgarsky2016benefits}
Matus Telgarsky.
\newblock Benefits of depth in neural networks.
\newblock \emph{arXiv preprint arXiv:1602.04485}, 2016.

\bibitem[Wallach et~al.(2015)Wallach, Dzamba, and Heifets]{wallach2015atomnet}
Izhar Wallach, Michael Dzamba, and Abraham Heifets.
\newblock Atomnet: A deep convolutional neural network for bioactivity
  prediction in structure-based drug discovery.
\newblock \emph{arXiv preprint arXiv:1510.02855}, 2015.

\bibitem[Zoran and Weiss(2012)]{Zoran:2012wu}
Daniel Zoran and Yair Weiss.
\newblock {"Natural Images, Gaussian Mixtures and Dead Leaves".}
\newblock \emph{Advances in Neural Information Processing Systems}, pages
  1745--1753, 2012.

\end{thebibliography}
}

\clearpage
\appendix

\section{Existence of Covering Templates} \label{app:cover_temp}

In this paper we analyze the expressiveness of networks, \ie the functions they can realize, through the notion of \emph{grid tensors}.
Recall from sec.~\ref{sec:nets2tens} that given \emph{templates} $\x^{(1)}\ldots\x^{(M)}\in\R^s$, the grid tensor of a score function $h_y:(\R^s)^N\to\R$ realized by some network, is defined to be a tensor of order $N$ and dimension $M$ in each mode, denoted $\A(h_y)$, and given by eq.~\ref{eq:grid_tensor}.
In particular, it is a tensor holding the values of $h_y$ on all instances $X=(\x_1,\ldots,\x_N)\in(\R^s)^N$ whose \emph{patches} $\x_i$ are taken from the set of templates $\{\x^{(1)}\ldots\x^{(M)}\}$ (recurrence allowed).
Some of the claims in our analysis (sec.~\ref{sec:analysis}) assumed that there exist templates for which grid tensors fully define score functions.
That is to say, there exist templates such that score function values outside the exponentially large grid $\{X_{d_1{\ldots}d_N}:=(\x^{(d_1)},\ldots,\x^{(d_N)}):d_1{\ldots}d_N\in[M]\}$ are irrelevant for classification.
Templates meeting this property were referred to as \emph{covering} (see sec.~\ref{sec:analysis:temp_rep_funcs}).
In this appendix we address the existence of covering templates.

If we allow $M$ to grow arbitrarily large then obviously covering templates can be found.
However, since in our construction $M$ is tied to the number of channels in the first (representation) layer of a network (see fig.~\ref{fig:convnet}), such a trivial observation does not suffice, and in fact we would like to show that covering templates exist for values of $M$ that correspond to practical network architectures, \ie~$M\in\Omega(100)$.
For such an argument to hold, assumptions must be made on the distribution of input data.
Given that ConvNets are used primarily for processing natural images, we assume here that data is governed by their statistics.
Specifically, we assume that an instance $X=(\x_1,\ldots,\x_N)\in(\R^s)^N$ corresponds to a natural image, represented through $N$ image patches around its pixels: $\x_1{\ldots}\x_N\in\R^s$.

If the dimension of image patches is small then it seems reasonable to believe that relatively few templates can indeed cover the possible appearances of a patch.
For example, in the extreme case where each patch is simply a gray-scale pixel ($s=1$), having $M=256$ templates may provide the standard $8$-bit resolution, leading grid tensors to fully define score functions by accounting for all possible images.
However, since in our construction input patches correspond to the receptive field in the first layer of a ConvNet (see fig.~\ref{fig:convnet}), we would like to establish an argument for image patch sizes that more closely correlate to typical receptive fields, \eg~$5{\times}5$.
For this we rely on various studies (\eg~\cite{Zoran:2012wu}) characterizing the statistics of natural images, which have shown that for large ensembles of images, randomly cropped patches of size up to $16{\times}16$ may be relatively well captured by Gaussian Mixture Models with as few as $64$ components.
This complies with the common belief that there is a moderate number of appearances taken by the vast majority of local image patches (edges, Gabor filters \etc).
That is to say, it complies with our assumption that covering templates exist with a moderate value of $M$.
We refer the reader to~\cite{\expresstensors} for a more formal argument on this line.

\end{document}